\documentclass{article}

\usepackage{microtype}
\usepackage{graphicx}
\usepackage{subcaption}
\usepackage{booktabs} % for professional tables

\usepackage{hyperref}

\usepackage[preprint]{icml2026}

\usepackage{mathtools}

\usepackage{amsmath,amsfonts,bm}

\newcommand{\inv}{^{\raisebox{.2ex}{$\scriptscriptstyle-1$}}}

\newcommand{\sminus}{\scalebox{0.8}[1.0]{$-$}}

\definecolor{mygreen}{RGB}{85,168,104}  % ~ Seaborn green
\definecolor{mygray}{RGB}{120,120,120}

\definecolor{myblue}{HTML}{4A90E2}
\definecolor{myred}{HTML}{D0021B}

\newcommand{\cl}[1]{\!#1\!}
\newcommand{\eq}{\cl{=}}

\newcommand{\clin}{\cl{\in}}
\newcommand{\inR}[1][]{\clin \sR^{#1}}
\newcommand{\inRR}[2]{\clin \sR^{#1\times#2}}
\newcommand{\deriv}[2]{\tfrac{d#1}{d#2}}
\newcommand{\pderiv}[2]{\tfrac{\partial#1}{\partial#2}}

\newcommand{\mTm}[1]{#1^{\top}\!#1}

\newcommand{\keypoint}[1]{\vspace{2pt}\textbf{#1}:\ }

\def\appref#1{Appendix~\ref{#1}}
\def\thmref#1{Thm.~\ref{#1}}

\DeclareMathOperator{\Diag}{Diag}

\def\Figref#1{Fig.~\ref{#1}}                             % my override

\def\Secref#1{§\ref{#1}}
\def\eqref#1{equation~\ref{#1}}
\def\Eqref#1{Eq.~\ref{#1}}                                  % my override
\def\1{\bm{1}}

\def\rx{{\textnormal{x}}}

\def\rz{{\textnormal{z}}}

\def\vr{{\bm{r}}}

\def\vu{{\bm{u}}}
\def\vv{{\bm{v}}}

\def\vz{{\bm{z}}}

\def\mD{{\bm{D}}}
\def\mE{{\bm{E}}}

\def\mH{{\bm{H}}}
\def\mI{{\bm{I}}}
\def\mJ{{\bm{J}}}

\def\mM{{\bm{M}}}

\def\mR{{\bm{R}}}
\def\mS{{\bm{S}}}

\def\mU{{\bm{U}}}
\def\mV{{\bm{V}}}
\def\mW{{\bm{W}}}
\def\mX{{\bm{X}}}

\DeclareMathAlphabet{\mathsfit}{\encodingdefault}{\sfdefault}{m}{sl}
\SetMathAlphabet{\mathsfit}{bold}{\encodingdefault}{\sfdefault}{bx}{n}
\newcommand{\tens}[1]{\bm{\mathsfit{#1}}}

\def\tH{{\tens{H}}}

\def\tL{{\tens{L}}}

\def\gM{{\mathcal{M}}}
\def\gN{{\mathcal{N}}}

\def\gV{{\mathcal{V}}}

\def\gX{{\mathcal{X}}}

\def\gZ{{\mathcal{Z}}}

\def\sR{{\mathbb{R}}}

\newcommand{\E}{\mathbb{E}}

\newcommand{\R}{\mathbb{R}}

\newcommand{\Var}{\mathrm{Var}}

\usepackage[nameinlink,noabbrev,capitalize]{cleveref} % optional

\usepackage[textsize=tiny]{todonotes}

\usepackage{url}
\usepackage{amssymb, amsmath, amsthm}                % for multiline equations
\usepackage{amsfonts}       % blackboard math symbols
\usepackage{dsfont}
\usepackage{thmtools,thm-restate}
\usepackage{enumitem}       % reduce list indentation
\usepackage{relsize}

\usepackage{dblfloatfix}
\usepackage{graphicx}               % for scalebox

\usepackage{booktabs}       % professional-quality tables
\usepackage{bm}
\usepackage{microtype}      % microtypography
\usepackage{wrapfig}
\usepackage[font=small]{caption}
\usepackage{subcaption}     % required for figures (e.g. graphs) to be put side-by-side

\usepackage{tikz}
\usetikzlibrary{arrows.meta,calc,positioning,decorations.pathmorphing}
\usepackage{tikz}
\usetikzlibrary{calc,arrows.meta,decorations.markings}
\definecolor{myblue}{RGB}{33,113,181}
\definecolor{myred}{RGB}{200,54,69}
\definecolor{mygreen}{RGB}{58,175,76}

\usetikzlibrary{matrix}
\usetikzlibrary{automata}

\usepackage{pgfplots}
\pgfplotsset{compat=1.18}

\usepackage{multirow}
\usepackage[title]{appendix}
\usepackage[hang,flushmargin]{footmisc}
\usepackage{cases}
\usepackage{pst-node}

\graphicspath{{figures/}}        % must go after "graphicx"

\theoremstyle{plain}
\newtheorem{theorem}{Theorem}[section]

\newtheorem{corollary}[theorem]{Corollary}
\theoremstyle{definition}
\newtheorem{definition}{Definition D\!}
  \newtheorem{condition}{Condition C\!} \theoremstyle{remark}
\newtheorem{remark}[theorem]{Remark}

\usepackage[textsize=tiny]{todonotes}

\icmltitlerunning{Disentanglement as Identifiable Pushforward Factorisation}

\begin{document}

\twocolumn[
  \icmltitle{Disentanglement as Identifiable Pushforward Factorisation}


  \icmlsetsymbol{equal}{*}

  \begin{icmlauthorlist}
    \icmlauthor{Carl Allen}{yyy}
    % \icmlauthor{Firstname2 Lastname2}{equal,yyy,comp}
    % \icmlauthor{Firstname3 Lastname3}{comp}
  \end{icmlauthorlist}

  \icmlaffiliation{yyy}{Centre for Data Science, \'Ecole Normale Sup\'erieure, Paris, France}

  \icmlcorrespondingauthor{Carl Allen}{carl.allen@ens.fr}

  \icmlkeywords{Machine Learning, ICML}

  \vskip 0.3in
]



\printAffiliationsAndNotice{}  % no special notice (required even if empty)

\begin{abstract}
    % \vspace{-6pt}
    We characterise \emph{disentanglement} in smooth generative pushforward models, such as in VAEs and GANs.
    For a generator/decoder $g\cl:\mathcal{Z}\!\to\!\mathcal{X}$ and factorised prior $p(z)\eq\prod_i p_i(z_i)$, we define disentanglement as \emph{factorisation of the pushforward density} $p_\mu\eq g_\#p$ into one–dimensional ``seam'' factors, where each latent dimension controls an independent generative factor of the data. We prove that $p_\mu$ factorises according to the SVD of $g$'s Jacobian; that disentanglement equates to two conditions on $g$ (C1-C2); and that under those conditions the seam factors are \emph{identifiable}, up to permutation and sign.
    % These results are general to smooth pushforwards as in VAEs and GANs.
    In the particular case of Gaussian ($\beta$-)VAEs, we show via an identity how diagonal posteriors promote C1-C2, in expectation, explaining why disentanglement arises modulated by $\beta$.
    Experiments illustrate this mechanism on Gaussian data, dSprites, and CelebA.
    \vspace{-12pt}
\end{abstract}

% \vspace{-4pt}

\section{Introduction}

\label{sec:intro}

\vspace{-4pt}
A generative latent variable model is said to be \textit{disentangled} if varying a single latent co-ordinate changes a single semantic aspect of generated samples, e.g.\ an object's position or the facial expression in an image.
Variational Autoencoders (\textbf{VAE}s, \citet{kingma2013auto, rezende2014stochastic}) and variants \citep{betavae, kim2018disentangling, tcvae} are often observed to disentangle, which is intriguing since VAEs are not knowingly designed to achieve it, and potentially useful, e.g.\ for controlled data generation. 
Related phenomena are observed in Generative Adversarial Networks (GANs) \citep[][]{gan} and diffusion models \citep{stablediff, pandeydiffusevae, zhang2022unsupervised, yang2023disdiff}.

Though lacking a formal definition, disentanglement is often associated with identifying \textit{generative factors} of the data \citep{bengio2013representation}. Hence understanding disentanglement, and how it arises ``for free'' in VAEs, is of interest in many areas of machine learning, to its interpretability and to our very understanding of the data itself. It may also enable data to be disentangled reliably in domains where it is less easy to perceive, e.g.\ gene sequence or protein modelling.

Research into ($\beta$-)VAEs attributes their ability to disentangle to the use of \textit{diagonal} posterior covariances, commonly chosen for computational efficiency \citep{betavae, burgess2018understanding, rolinek2019variational, kumar2020implicit}. 
Approximate relationships suggest that diagonal covariances promote \textit{column-orthogonality} in the decoder's Jacobian, a property empirically associated with disentanglement \citep{gresele2021independent}. 
In GANs, disentangled features have been identified via the SVD of the generator's Jacobian \citep{ramesh2018spectral}, or induced by regularising generator derivatives: \citet{wei2021orthogonal} regularise Jacobian orthogonality,
% (closely aligned with C1), 
while \citet{peebles2020hessian} penalise mixed Hessian interactions.
% (closely related to the Hessian-side mechanism behind C2).
% 
We connect these findings via a general principled distribution-level explanation of disentanglement.

Our core results are \emph{model-agnostic} statements general to smooth pushforwards under a deterministic function
% generator/decoder 
$g$: 
\vspace{-4pt}
\begin{itemize}[leftmargin=0.3cm, itemsep=0pt, topsep=0pt]    
    \item a distributional definition of disentanglement in terms of independent factors/components (\Cref{def:disentanglement}), 
\vspace{-2pt}
    \item a canonical factorisation of the pushforward density over the manifold via the SVD of $g$'s Jacobian  (\Cref{thm:seam-factorisation-general}),
\vspace{-2pt}
    \item an \emph{iff} characterisation of disentanglement via two conditions on $g$'s derivatives (C1-C2, \Cref{thm:non_linear}), and 
\vspace{-2pt}
    \item proof that \textit{independent factors are identifiable} up to natural symmetries (\Secref{sec:identifiability}).
\end{itemize}
% \vspace{-4pt}
In the special case of Gaussian ($\beta$-)VAEs, we show from an \textit{exact} identity that diagonal posteriors encourage C1-C2 in aggregate / in expectation over posteriors (\Figref{fig:disentanglement}), modulated by $\beta$ (\Secref{sec:cleaning_up_covariances}).

Overall, we extend prior understanding of disentanglement, and link VAE theory with empirical findings via a rigorous yet intuitive, distribution-level definition of disentanglement, mathematically formalising this well-known phenomenon.

% Specifically, we:
% \vspace{-4pt}
% \begin{itemize}[leftmargin=0.3cm, itemsep=1pt]    
%     \item formally define disentanglement as factorising the density over a manifold into independent 1-D \textit{seam} factors, each the push-forward of the density over an axis-aligned latent path (D\ref{def:disentanglement}, \Figref{fig:disentanglement});
%     % 
%     % \vspace{-1pt}
%     \item show that $\beta$ of a $\beta$-VAE \citep{betavae} effectively controls the model variance $\Var[x|z]$, which explains why it is found to enhance disentanglement and mitigate ``posterior collapse'' (\Secref{app:beta});
%     % 
%     % \vspace{-1pt}
%     \item identify constraints that are necessary and sufficient for disentanglement, and show that diagonal posterior covariances induce them, in aggregate and in expectation  (\Secref{sec:cleaning_up_covariances}); and
%     % 
%     % \vspace{-1pt}
%     \item prove that if a data distribution has ground truth independent factors, those factors can be \textit{identified} by a Gaussian VAE, resolving the ``unidentifiability'' in non-linear ICA with Gaussian priors (\Secref{sec:identifiability}).
%     % \vspace{-2pt}
% \end{itemize}

\begin{figure*}[t]
\centering
% \vspace{-4pt}
\input{figures/figure_one}
\vspace{-6pt}
\caption{
\textbf{Disentanglement illustrated for VAEs with full vs diagonal posteriors:} 
Right singular vectors $\vv_i\clin\gZ$ of the decoder's Jacobian (blue arrows) define (\textit{s.v.}) paths in latent space (dashed blue); 
left singular vectors $\vu_i$ (red arrows) define \textit{seams} in data space (dashed red). 
1-D densities over seams \underline{factorise} the manifold density $p_\mu$.
(\textbf{\textit{left}}) For VAEs with full posteriors, s.v.\ paths are arbitrary, hence the image of an axis‑traversal (green) need \emph{not} follow a seam. (\textbf{\textit{right}}) For VAEs with diagonal posteriors, s.v.\ paths are \underline{axis‑aligned}, hence axis traversal images follow seams, over which 1-D densities are \underline{independent} factors of $p_\mu$, achieving \textit{disentanglement}, (D\ref{def:disentanglement}).}
\label{fig:disentanglement}
\vspace{-12pt}
\end{figure*}

\section{Background}
\label{sec:bg}

\vspace{-4pt}

\keypoint{Notation} 
For data $x\clin\gX\cl\doteq\sR^m$ with density $p(x)$, and latent variables $z\clin\gZ\cl\doteq\sR^d$ ($d\cl\leq m$), we consider functions $f:\gZ\to\gX$ that are continuous, injective and differentiable almost everywhere (\textbf{\emph{c.i.d.a.e.}}), with image
$\gM_f \eq \{f(z)\,|\,z\in\gZ\}\subseteq\gX$.
% 
% which
% 
% A function $f\cl: \gZ\cl\to\gX$ is \textbf{\emph{c.i.d.a.e.}}\ if it is continuous, injective and differentiable almost everywhere (\textit{a.e.}), e.g. ReLU networks.
% 
% define a $d$-dimensional \textit{manifold}  $\gM_f \eq \{f(z)\,|\,z\clin\gZ\}$ embedded in $\gX$ (see \Figref{fig:seams}).
% 
If $f$ is differentiable at $z$, $\mJ_z$ denotes its Jacobian evaluated at $z$ ($[\mJ_z]_{ij} \eq \pderiv{x_i}{z_j}$) with singular value decomposition (\textbf{SVD}) $\mJ_z\eq\mU_z\mS_z\mV_z^\top$
    ($\mTm{\mU_z} \eq \mI$,    $\mTm{\mV_z} \eq \mV_z\mV_z^\top\! \eq \mI$).\footnote{
    To lighten notation, explicit dependence of $\mU, \mV, \mS, \vu_i, \vv_i, s_i$ on $z$ is often suppressed where context is clear.}
Let  $s_{i}\cl\doteq\mS_{ii}$ denote the $i^{th}$ singular value, and $\vu_{i}$/$\vv_{i}$ the $i^{th}$ left/right singular vectors (columns of $\mU$/$\mV$).
Where $\mJ_z$ exists and has full column rank, $\gM_f$ is locally a $d$-dimensional embedded manifold  (see \Figref{fig:seams}).
By injectivity, $f$ defines a bijection between $\gZ$ and $\gM_f$.
% ; and $\mJ_{z}$ has full-rank, where defined. 

\vspace{-2pt}
\keypoint{Latent Variable Model (LVM)}
We consider generative models with parameters $\theta$ and independent latent variables:
% $
\vspace{-4pt}
\begin{align}
\label{eq:lvm}
p_\theta(x) \eq \int_z p_\theta(x|z)p(z),\quad p(z)\eq\prod_ip(z_i)\ .
\end{align}
% $ 

\vspace{-12pt}
The \textbf{generative function} $g \cl: \gZ \cl\to \gX$  is the map  $g(z)\eq\E[x|z]$ from latent space to the \textbf{mean manifold} $\gM_g\eq\{g(z)\}\cl\subseteq\gX$, with push-forward density $p_{\mu} \cl\doteq g_{\#} p_z$ (\textbf{manifold density}). \Eqref{eq:lvm} admits the deterministic case $p_\theta(x|z) \eq \delta_{x-g(z)}$, where $p_\theta(x) \eq p_\mu(x)\mathds{1}_{x\in\gM_g}$ (\textbf{deterministic LVM}).
% \footnote{
% % 
% Note that where $p(x)$ is given by adding Gaussian noise to $p_\mu$, or convolving with a Gaussian kernel, two convolved densities match if and only if their manifold densities match \citep[e.g.][]{khemakhem2020variational}, 
% allowing us to focus on the manifold densities.}

\vspace{-2pt}
\keypoint{Variational Autoencoder (\textbf{VAE})}
A VAE learns parameters of an LVM by maximising the ELBO (where $\beta \eq1$),
\vspace{-4pt}
\begin{align}
    \label{eq:elbo}
    \!\!\int_x p(&x)\log p_\theta(x)
        \nonumber\\  \geq \ &   
        \int_x \!p(x) \!\int_z\! q_\phi(z|x)
            \Big(\!
                \log p_\theta(x|z)  - \beta\log\tfrac{q_\phi(z|x)}{p(z)} 
            \!\Big).
\end{align}

\vspace{-10pt}
A \textit{decoder} network $d(z)$ parameterises likelihoods $p_\theta(x|z)$;
an \textit{encoder} network parameterises Gaussian approximate posteriors $q_\phi(z|x) = \gN(z;e(x),\Sigma_x)$, with \textit{diagonal} covariances $\Sigma_x$. The prior $p(z)$ is standard Gaussian.
We call a VAE with Gaussian likelihoods $p_\theta(x|z) \cl\doteq \gN(x; d(z),\sigma^2\mI)$ a \textbf{Gaussian VAE}. A Gaussian VAE with a linear decoder $d(z) \eq \mD z,\,\mD\inRR{m}{d}$, is called a  \textbf{linear VAE (LVAE)}.

\keypoint{Generative Adversarial Network (\textbf{GAN})} A GAN learns a deterministic LVM parameterised by a \textit{generator} network $g$, by minimising the Jensen-Shannon divergence between $p_\theta(x)$ and $p(x)$, as approximated by a classifier.

\vspace{-2pt}
\keypoint{Disentanglement} While not formally defined, disentanglement refers to when a trained model identifies semantic features with distinct latent variables $z_i$, such that varying a latent co-ordinate causes generated samples to differ by a single feature \citep[][]{bengio2013representation, betavae, ramesh2018spectral, rolinek2019variational}.
% , shuweakly}.
% 
This appears to relate closely to \textit{Independent Component Analysis} (\textbf{ICA}), which aims to identify \textit{statistically independent components} of data generated under an LVM (\Eqref{eq:lvm}).

Surprisingly, VAEs ($\beta \eq 1$) often exhibit disentanglement, and even more so with $\beta\cl>1$ \citep[$\bm{\beta}$\textbf{-VAE},][]{betavae, burgess2018understanding}, although with reduced generative quality, e.g.\ blurrier images. A standard GAN does not automatically disentangle, but can be induced to by diagonalising generator derivatives \citep{peebles2020hessian, wei2021orthogonal}. We theoretically justify and unify these observations.

\label{sec:ppca}
\vspace{-2pt}
\textbf{Linear LVM}:\ Linear LVMs are studied in
Probabilistic PCA (PPCA) \citep{ppca}, i.e.\ \Eqref{eq:lvm} where
% \vspace{-10pt}
\vspace{-2pt}
\begin{align}
    \label{eq:ppca_assumps}
    \!\!\!p(x|z) = \gN(x; \mW z, \sigma^2\mI),\ \quad\ 
    p(z) &= \gN(z;\,\bm{0},\mI)\,,
    % \qquad\
    % \nonumber\\
    % \epsilon \sim p_\sigma(\epsilon) = \gN(\epsilon;\,\bm{0},\sigma^2\mI)\,,
\end{align}

\vspace{-10pt}
$\mW \!\inRR{m}{d}$, $\sigma \inR$.\footnote{
    For simplicity we assume that data is centred, equivalent to including a mean parameter \citep{ppca}.}
Here, posteriors $p(z|x)$ and MLE parameters $\mW_*$ are known analytically:
\vspace{-2pt}
\begin{gather}
    \!\! p(z|x) 
    % 
    % %   *** KEEP WORKINGS BELOW ***
    % 
    % = \gN(z;\, \mM\inv\mD^\top\! x,\, \sigma^2\mM\inv)  
    % \qquad\text{where}\quad  \mM = \mD^\top\!\mD + \sigma^2\mI\ .
    % \\
    % = \gN(z;\, (\mD^\top\!\mD + \sigma^2\mI)\inv\mD^\top\! x,\, \sigma^2(\mD^\top\!\mD + \sigma^2\mI)\inv)
    % \\
    % = \gN(z;\, \tfrac{1}{\sigma^2}(\mI + \tfrac{1}{\sigma^2}\mD^\top\!\mD)\inv\mD^\top\! x,\, (\mI + \tfrac{1}{\sigma^2}\mD^\top\!\mD )\inv)
    % \\
    % \mM_{new}= (\mI + \tfrac{1}{\sigma^2}\mD^\top\!\mD)\inv = \sigma^2\mM_{old}\inv
    % \\
    % 
    % %   *** KEEP WORKINGS ABOVE ***
    % 
    = \gN(z; \tfrac{1}{\sigma^2}\mM \mW^\top\! x,  \mM),
    \ \     
    \mM\inv\! \eq \mI \cl+ \tfrac{1}{\sigma^2}\mTm{\mW}
    \nonumber\\ 
        \mW_* = \mU_{\!\mathsmaller\mX}(\mathbf{\Lambda}_{\!\mathsmaller\mX} \,\sminus\, \sigma^2\mI)^{^1\!/\!_2}\mR
    \label{eq:ppca-posterior}
    % \vspace{-2pt}
\end{gather}

\vspace{-10pt}
where $\mathbf{\Lambda}_{\!\mathsmaller\mX} \!\inRR{d}{d}$, $\mU_{\!\mathsmaller\mX} \!\inRR{m}{d}$ contain the largest eigenvalues and respective eigenvectors of data covariance $\mX\mX^\top\!$; and $\mR \inRR{d}{d}$ is orthonormal ($\mTm{\mR} \eq\mI$). As $\sigma^2 \!\to 0$, $\mW_*$ approaches the SVD of the data matrix 
    $\mX \! \eq \mU_{\!\mathsmaller\mX}\mathbf{\Lambda}_{\!\mathsmaller\mX}^{^{_\mathsmaller{_1\!}}\!\mathsmaller{/}\!_\mathsmaller{2}}\mV_{\!\mathsmaller\mX}^\top\! \inRR{m}{n}\!$, up to $\mV_{\!\mathsmaller\mX}$ (classical PCA). 
The model is \textit{unidentified} since $\mW_*$ has uncountably infinite solutions (due to $\mR$).
    % .
% \textcolor{red}{Although $\mW_*$ \textit{can} be computed analytically, it can also be learned by maximising \Eqref{eq:elbo} via the EM-algorithm: iteratively set $p_\theta(x|z) \eq \gN(x;\mD z,\sigma^2\mI)$, compute the optimal posterior (\Eqref{eq:ppca-posterior}) and maximise w.r.t.\ $\mD$ 
% (termed \textbf{\ppcaEM}).}

\begin{figure*}[t!]
    \vspace{-4pt}
    \centering
    \begin{subfigure}{.35\textwidth}
        \centering
        \hspace{-20pt}
        \includegraphics[width=0.85\linewidth]{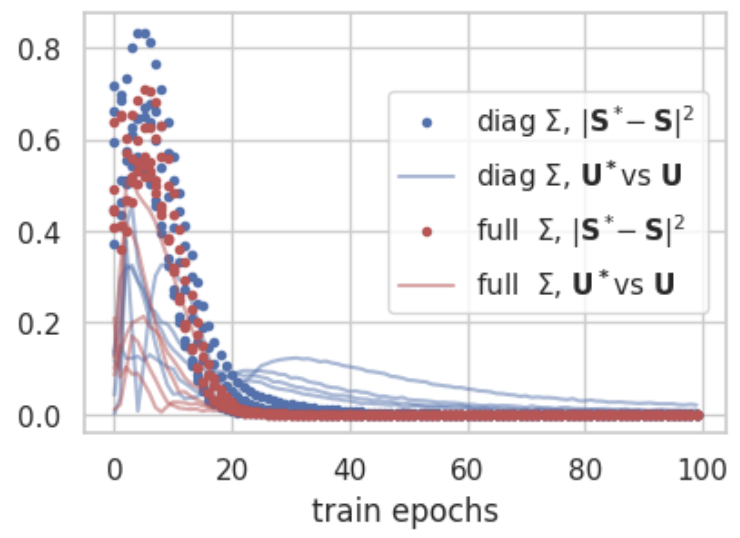}
        % \scalebox{0.99}{ \includegraphics[width=\linewidth]{linear2.png} }
    \end{subfigure}%
    \begin{subfigure}{.35\textwidth}
        \centering
        \hspace{-30pt}
        % \scalebox{1.}{
        \includegraphics[width=0.9\linewidth]{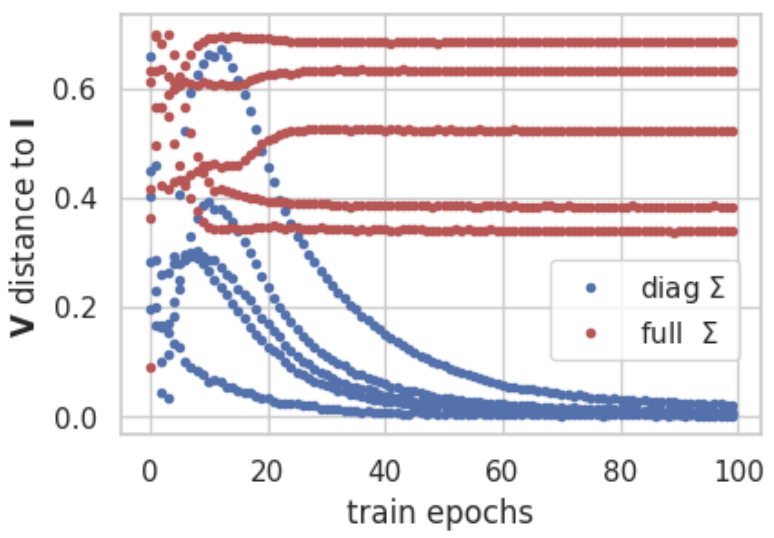}
        % }
    \end{subfigure}%
    \begin{subfigure}{.3\textwidth}
        \centering
        \hspace{-30pt}
        % \scalebox{1.}
        {\includegraphics[width=0.9\linewidth]{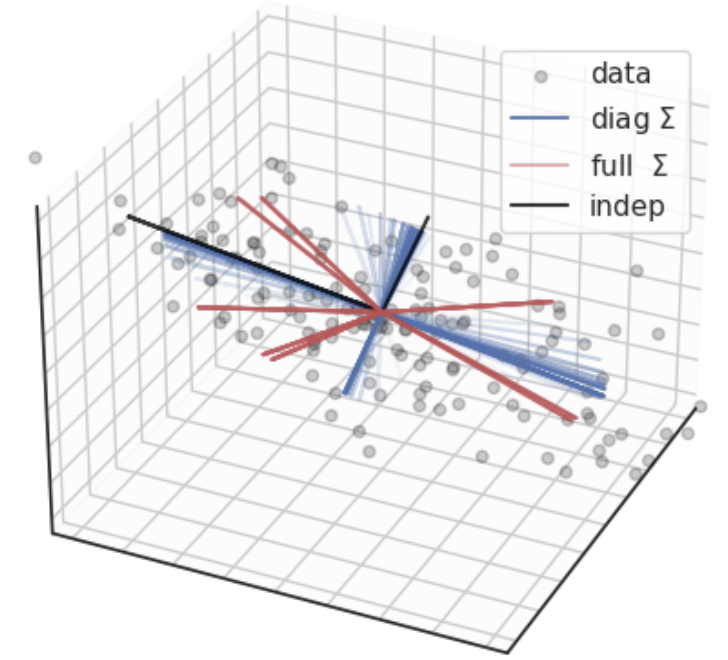}}
        % }
    \end{subfigure}%
    \vspace{-4pt}
    \caption{
        \textbf{An LVAE breaks rotational symmetry.} 
        (\textit{\textbf{l}}) 
            both full and diagonal posterior VAEs fit the data, learning ground truth parameters $\mU$, $\mS$ (all losses $\cl\to 0$);
        (\textit{\textbf{c}}) 
            only with diagonal $\Sigma$ do right singular vectors $\vv^i$ of $\mD$ align with standard basis vectors $\vz_i$ ($\mV\cl\to\mI$, blue plots); 
        (\textit{\textbf{r}}) 
            % image of $\vz_i$: 
                full-$\Sigma$ VAEs map latent axes $\vz_i$ to arbitrary directions (red), 
                but diagonal-$\Sigma$ VAEs learn (later epochs darker) to map $\vz_i$ to independent components of the data (black, i.e.\ blue $\to$ black).
    }
    \label{fig:linear_results}
    \vspace{-10pt}
\end{figure*}

\vspace{-2pt}
Given the same linear LVM, an LVAE
% \label{sec:lvaes}
 approximates the solution by maximising the ELBO with $p_\theta(x|z) \eq \gN(x;\mD z,\sigma^2\mI)$ and $q_\phi(z|x) \eq \gN(z;\mE x,\Sigma)$.
Interestingly, an LVAE with  \textit{diagonal} $\Sigma$ breaks the rotational symmetry of PPCA \citep{lucas2019don}. From \Eqref{eq:ppca-posterior} we have
\vspace{-4pt}
\begin{align}
    \label{eq:breaking_symmetry_linear}
    \Sigma_{*}
    % \, =\        \mM_{*}
    \, \overset{(\ref{eq:ppca-posterior}\textit{i})}=\  
        (\mI \cl+ \tfrac{1}{\sigma^2}\mW_*^\top\!\mW_*)\inv
    \ \overset{(\ref{eq:ppca-posterior}\textit{ii})}=\   \sigma^2\mR^\top\!\mathbf{
    \Lambda}_{\!\mathsmaller\mX}\inv\mR
    \ ,
\end{align}

\vspace{-10pt}
so for $\Sigma$ to be both optimal \textit{and} diagonal, $\mR$ must belong to a finite set of signed permutations, thus the optimal decoder $\mD_*\! \eq\, \mU_{\!\mathsmaller\mX}(\mathbf{\Lambda}_{\!\mathsmaller\mX} \sminus\, \sigma^2\mI)^{^1\!/\!_2}$
is unique \textit{up to permutation and sign} (\textbf{P\&S}) (see \Figref{fig:linear_results}).
We claim that this result  of diagonal covariances is in fact \textit{disentanglement} in the linear case.
% (\Secref{sec:disentanglement}).

\section{Disentanglement}\label{sec:disentanglement}
\vspace{-4pt}

% Below we define disentanglement, demonstrate it in the linear case justifying our LVAE claim (\Secref{sec:bg}), and then characterise when it occurs in the general non-linear case (\Secref{sec:orthog_to_disent}).
% % (see \Figref{fig:disentanglement} for intuition.) 

We define disentanglement in terms of statistical independence. Random variables/co-ordinates $x_i$ are independent \textit{iff} their joint density \textit{factorises}, e.g.\ $p(x)\eq \prod p(x_i)$. Similarly, a pushforward density over a manifold $p_\mu$ may factorise as a product of independent 1‑D factors over the manifold (rather than ambient co-ordinates $x_i$). We say $p_{\mu}$ is disentangled when such a factorisation exists and each factor maps to/from a distinct latent $z_i$ (\textit{axis-alignment}), capturing the notion that changing one factor leaves the others unaffected.

\begin{definition}[\textbf{Disentanglement}]\label{def:disentanglement}

For $g\cl:\gZ\to\gX$ \emph{c.i.d.a.e.}, we say $p_\mu$ is \emph{disentangled} at $z\clin\gZ$
if 
there exist  1-D densities $\{f_i\}$ such that $p_\mu$ \underline{factorises} as
\vspace{-7pt}
\begin{equation}
    \label{eq:disentanglement}
    p_\mu\big(g(z)\big) = \prod_{i=1}^d f_i\big(u_i(z)\big)\,,
    % \quad 
    % \forall z\clin\gZ
    %\big(g(z)\big)
\end{equation}

\vspace{-12pt}
where: 
(a) each factor $f_i$ is the \emph{1-D push-forward} of $p(z_i)$ along the \underline{axis-aligned} line obtained by moving in the $i^{th}$ latent coordinate $z_i$ (with all others fixed); 
(b) $u_i$ is the co-ordinate of $g(z)$ along the image of that line; and  
(c) random variables $\{u_i(z)\}$ are mutually \underline{independent} under $z\cl\sim p(z)$.
\end{definition}

\vspace{-2pt}
To illustrate we consider the linear case: fitting an LVAE with decoder $\mD$ and diagonal posterior covariance to data generated by a linear LVM (\Eqref{eq:ppca_assumps}, $\mW\eq\mU\mS\mV^\top$) where $p(x)\eq \mathcal{N}(x;0,\mW\mW^\top\!+\sigma^2\mI)$.
% , mean manifold 
    % $\gM_d\cl\doteq\{\mu\eq \mD z \cl\mid z\clin\gZ$\} 
% 
% LVAE with diagonal posterior covariance (\Secref{sec:bg}), with decoder $d(z)\eq\mD z,\, \mD\inRR{m}{d}\!$, mean manifold 
%     $\gM_d\cl\doteq\{\mu\eq \mD z \cl\mid z\clin\gZ$\} 
% and Gaussian density 
%     $p_\mu \!\eq \gN(\mu; \bm{0},\mD\mD^\top)$ 
% (\Figref{fig:linear_results}, \textit{right}).
From \Secref{sec:bg}, the data covariance 
% (by \Eqref{eq:ppca-posterior}) 
and diagonal posteriors determine the optimal decoder $\mD_* \eq \mU\mS$ (up to P\&S). The claim is that the decoder $\mD_*$ disentangles $p_\mu \eq \mathcal{N}(0,\mW\mW^\top)$ per \Cref{def:disentanglement}.

\vspace{-2pt}
As a Gaussian, $p_\mu$ {factorises} as a product of independent 1-D Gaussians along eigenvectors $\vu_i$ of its covariance
    $\mW\mW^\top\!\eq \mU\mS^2\mU^\top\!$:\ 
    $p_\mu(x)\eq \prod_i \gN(u_i; 0,s_{i}^2)$, 
for \textit{feature co-ordinates}
    $u_i \cl\doteq \vu_i^{\top}\!x \clin\sR$. 
If $x\eq\mD_* z$,  
% yielding the required factorisation in \Eqref{eq:disentanglement}. 
then $u_i 
% \eq \vu_i^{\top}\!x
\eq \vu_i^{\top}\!\mD_* z \eq
s_i z_i$, hence
\vspace{-10pt}
\begin{itemize}[itemsep=0pt, leftmargin=0.6cm]
    \item $p_\mu$ \underline{factorises} as a product of \underline{\smash{independent}} push-forward densities 
        $f_i(u_i)\eq \gN(u_i; 0,s^{i2})$ (as in \Eqref{eq:disentanglement}); and
    \vspace{-2pt}
    \item the decoder maps each \underline{\smash{axis-aligned}} direction $z_i$ to a distinct feature co-ordinate $u_i$ with density $f_i$,
\end{itemize}
\vspace{-7pt}
satisfying D\ref{def:disentanglement}.
As expected, synthetic samples $x\eq\mD z$ generated by re-sampling $z_i$, holding $z_{j\ne i}$ constant, differ in a single principal component, or feature, $u_i$.
% , fitting the ``one latent controls one feature'' intuitive understanding of disentanglement.
% \vspace{-14pt}

\keypoint{Dropping diagonality} 
To highlight the role of \textit{diagonal} posteriors we consider \textit{full} posterior LVAEs where $\mR\cl\ne\mI$ in general (\Eqref{eq:ppca-posterior}). The above argument follows except that (a) columns $\vr_i\clin\gZ$ of $\mR$ map to independent $\vu_i$ directions in $\gX$; and (b) standard basis vectors in $\gZ$ map to $\vu_i^{\top}\!\mR$ directions, arbitrary with respect to independent features along $\vu_i$. Axis-aligned traversals in latent space thus correspond to varying \textit{entangled} combinations of (not distinct) features. 
This is shown empirically in \Figref{fig:linear_results} (see caption for details).

\section{Disentanglement $\Leftrightarrow$ Decoder Constraints}
\label{sec:orthog_to_disent}
\vspace{-4pt}

% \subsection{The Jacobian SVD: Singular Vector Paths and Data Seams}

The linear case shows that diagonal covariances cause disentanglement (D\ref{def:disentanglement}) by constraining right singular vectors of the decoder.
% such that $\mD\eq\mU\mS$ (up to permutation/sign). 
To extend this to \textit{non-linear} pushforwards, we consider in more detail the generative function's Jacobian.

\vspace{-2pt}
For a Jacobian $\mJ_z \eq \mU\mS\mV^\top\!$, singular vectors $\vv_i$, $\vu_i$ (columns of $\mV,\,\mU$)
respectively define local orthonormal bases for $\gZ$ at $z$ (\textbf{$\bm{\mV}$-basis}) and the tangent space to $\gM_g$ at $x\eq g(z)$ (\textbf{$\bm{\mU}$-basis}). 
If $v\cl\doteq\mV^\top \!z$ and  $u\cl\doteq \mU^\top \!x$ denote  a point $z$ and its image $x\eq g(z)$ in those bases, the chain rule gives intuition to the Jacobian SVD,  $\mJ_z \eq \mU\mS\mV^\top\! \eq \pderiv{x}{u}\pderiv{u}{v}\pderiv{v}{z}$:
    $\mU$ and $\mV^\top\!$ are local co-ordinate systems, and 
    $\mS \eq \deriv{u}{v}$ is the Jacobian of a map  $v\cl\mapsto u$ in those co-ordinates, under which \textit{only respective dimensions interact} ($\tfrac{\partial u_i}{\partial v_j} \eq 0$, $i\cl\ne j$).

\begin{figure*}[!t]
    \centering
    \vspace{-10pt}
    % \begin{subfigure}{0.5\textwidth}
        % \centering
        \includegraphics[width=0.5\linewidth]{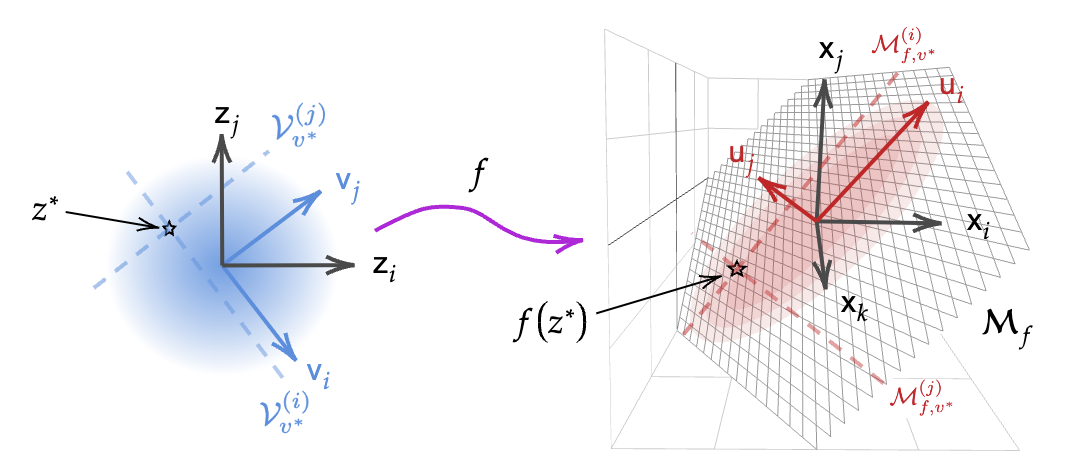} 
    % \end{subfigure}%
    \hfill
    % \begin{subfigure}{0.47\textwidth}
        % \centering
        \includegraphics[width=0.47\linewidth]{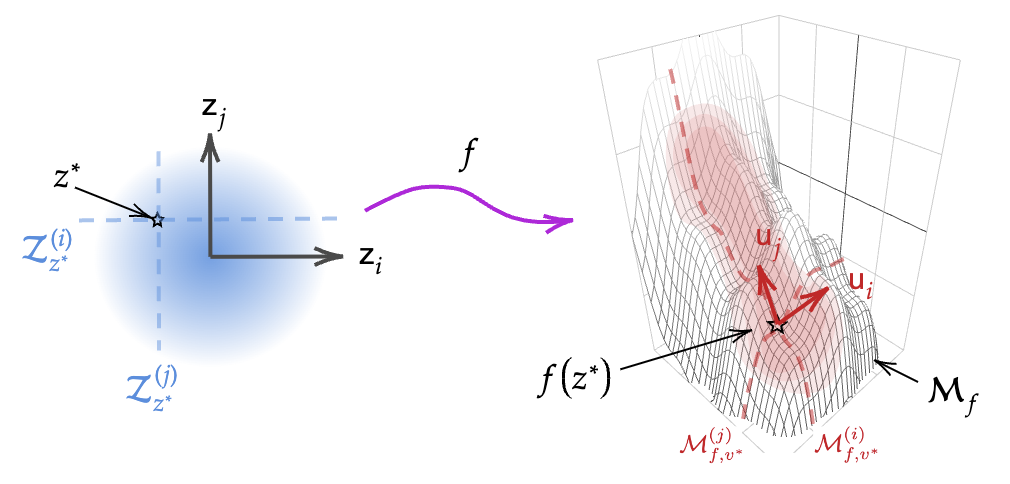} 
    % \end{subfigure}%
    \vspace{-4pt}
    \caption{
        \textbf{Pushing forward $\bm{p(z)}$, from singular vector paths to seams}:
    % 
        % \emph{s.v.\ paths} $\gV_{z^*}^i\!$ through $z^* \clin\gZ$   (\textcolor{myblue}{dash blue}) 
        %     follow right singular vectors $\vv^i$ of $\mJ_{z}$ (\textcolor{myblue}{solid blue}); 
        % \emph{seams} $\gM_{f,z^*}^i \!\cl \subseteq\gM_f$ through $f(z^*)$  (\textcolor{myred}{dash red}) follow left singular vectors $\vu^i$ (\textcolor{myred}{solid red}). 
        % S.v.\ paths map to seams.
        % 
        1-D marginals $p_i(z_i)$ over \emph{s.v.\ paths} $\gV_{z^*}^i$ (\textcolor{myblue}{dashed blue}) factorise $p(z)$; and push-forward to 1-D seam densities over \emph{seams} $\gM_{f,z^*}^i$ (\textcolor{myred}{dashed red}) that factorise $p_\mu$ (\Cref{thm:seam-factorisation-general}).
        (\textit{\textbf{l}}) 
            For linear $f$ \textit{without} \textbf{C1} (e.g.\ full-$\Sigma_x$ LVAE),
            $\gV_{z^*}^i$ are straight lines but need not axis-align (as observed in \Figref{fig:linear_results}(\textit{r})).
        (\textit{\textbf{r}}) 
            For c.i.d.a.e.\ $f$ satisfying \textbf{C1-C2}, $\gV_{z^*}^i$ are \underline{\smash{axis-aligned}} (by C1) and seam densities are \underline{\smash{independent}} components (by C2) that \underline{factorise} $p_\mu$, as required for disentanglement (D\ref{def:disentanglement}).
    }
    \vspace{-10pt}
    \label{fig:main_fig}
\end{figure*}

\vspace{-2pt}
\keypoint{Singular Vector Paths and Seams}
Directional derivatives $\mJ_z\vv^{i} \eq \mU\mS\mV^\top\!\vv^{i} \eq s^{i}\vu^{i}$
show that a small perturbation by a right singular vector $\vv^{i}$ at $z\clin\gZ$ maps under $g$ to a small perturbation in direction $\vu^{i}$ at $g(z)\clin\gM_g$. 
By extension, if a path in $\gZ$ follows $\vv^i$ at each point (as a vector field) its image on $\gM_g$ is a path following $\vu^i$. 
(Note that where the Jacobian is continuous, its SVD components are also.) Since columns of an SVD can be permuted, an order must be
fixed for paths over singular vectors to be well defined:

% \vspace{-2pt}
\begin{definition}[\textbf{Regular set and continuous SVD}]
\label{def:regular-set}
For \emph{c.i.d.a.e.}  $g:\gZ\to\gX$, the \emph{regular set} is defined as
\vspace{-4pt}
\begin{align*}
    \gZ_{\mathrm{reg}}
    \;\doteq\;
    \big\{z\clin\gZ \;\big|\; \mJ_z &\text{ exists, has full column rank, }
    \\ 
    &\text{and } s_{1}(z)>\cdots>s_{d}(z)>0 \big\}.
    % \vspace{-6pt}
\end{align*}
Vector fields $z\mapsto \vv_{i}(z)$, $z\mapsto \vu_{i}(z)$ and singular values $z\mapsto s_{i}(z)$ can be made continuous on each connected component of $\gZ_{\mathrm{reg}}$ by fixing the SVD $\mJ_z\eq\mU_z\mS_z\mV_z^\top$.\footnote{
    e.g.\ start from an arbitrary point, order singular values strictly decreasing and choose signs continuously.
    }\footnote{
    Restricting to $\gZ_{\mathrm{reg}}$ only excludes points where $\mJ_z$ is undefined or has repeated singular values; these edge cases can be avoided without affecting the results. All statements are made on a fixed connected component of $\gZ_{\mathrm{reg}}$.
    % % 
    }
\end{definition}

% \vspace{-8pt}

Paths following $i$-th singular vectors can now be defined: 
    \textbf{singular vector} (\textbf{s.v.}) \textbf{path} $\gV^i_{z}$ (Def.~\ref{def:sv-path}) follows $\vv_i$ in $\gZ$;
        % (\Figref{fig:seams}, blue dashed lines); 
        and 
    \textbf{seams}, $\gM^i_{g,z}$ (Def.~\ref{def:seams})  follow $\vu_i$ on the manifold. 
        % (\Figref{fig:seams}, red dashed lines).

\begin{figure}[!h]
    \vspace{-10pt}
    \centering
    \subfloat{\includegraphics[width=0.87\linewidth]{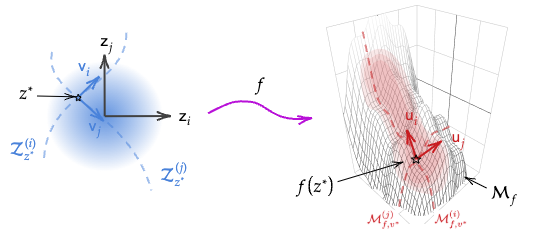}}%
    
    \vspace{-5pt}
    \captionsetup{font=small}
    \caption{
    % (\textbf{\textit{left}}) 
    %     \textbf{Empirical support for P\ref{prp:non-cancellation}}:
    %     % results of 
    %     \cite{rolinek2019variational} show that VAEs with diagonal $\Sigma_x$ have increased orthogonality in the decoder Jacobian (C1).
    % % 
    % (\textbf{\textit{right}}) 
        \textbf{Seam factorisation}:
        For c.i.d.a.e.\ $f\cl:\gZ\cl\to\gX$,
        % in \Cref{thm:seam-factorisation-general}, 
        with manifold $\gM_f\cl\subseteq\gX$, 
        \textit{s.v.\ paths} $\gV^k_{z^*}\cl\subseteq\gZ$ (\textcolor{myblue}{dashed blue})  following right singular vectors $\vv_{k}\!$ of Jacobian $\mJ_{z^*\!}$ at $z^*$ (\textcolor{myblue}{solid blue}),
        map to \textit{seams} $\gM_{f,z^*}^k\cl\subseteq \gM_f$  (\textcolor{myred}{dashed red}) following left singular vectors $\vu_{k}$ at $f(z^*)\clin\gX$ (\textcolor{myred}{solid red}). $\vz_k\clin\gZ$ are standard basis vectors.
        % in $\gZ$.
        }
    \vspace{-4pt}
    \label{fig:seams}
    % \vspace{-10pt}
\end{figure}

By construction, $g$ maps s.v.\ paths to seams (\Cref{lem:seams}),
% , extending that right singular vector perturbations map to distinct left singular vector perturbations
% (since $\mS$ is diagonal). 
hence 1-D densities over s.v.\ paths map  to 1-D pushforwards over seams. Thus if s.v.\ paths follow latent axes, and seam densities are independent factors of $p_\mu$, we have disentanglement (D\ref{def:disentanglement}). This occurs under two conditions:

\begin{condition}
    Right singular vectors $\mV_z$ of $\mJ_z$ are standard basis vectors, i.e.\ after axis relabeling/sign flips, $\mV_z=\mI$.
\end{condition}
\begin{condition}
    The matrix of partial derivatives of singular values $(\pderiv{s_i}{z_j})_{i,j}$ is diagonal, i.e.\ $\pderiv{s_i}{z_j}\!\eq 0$ for all $i\cl\neq j$.
\end{condition}

\vspace{-2pt}
\begin{restatable}[\textbf{Factorisation over seams}]{lemma}{thmseamfactorisation}
\label{thm:seam-factorisation-general}
Let $g\cl:\gZ\!\to\!\gX$ be \emph{c.i.d.a.e.}\  and the prior factorise as $p(z)=\prod_{i=1}^d p_i(z_i)$.
% (e.g.\ standard Gaussian).
Then, 
the manifold density $p_\mu$ on $\gM_g$
factorises as
\vspace{-8pt}
\begin{equation}
    \label{eq:pf-general-product}
    p_\mu\!\big(g(z)\big)
    \;=\;
    \prod_{i=1}^d \tfrac{p_i(z_i)}{s_{i}(z)}\,,
    \quad 
    \text{for every } z\clin\gZ_{\mathrm{reg}}.
    \vspace{-8pt}
\end{equation}
Moreover, under C1, each factor $
\tfrac{p_i(z_i)}{s^{i}(z)}$ is the 1-D density  over  the $i$-th seam $\gM^i_{g,z}$ at $x\eq g(z)$, obtained by pushing forward the 1-D marginal $p_i(z_i)$ over $\gV^i_{z}$, the $i$-th s.v.\ path though $z$ parallel to the $i^{\text{th}}$ latent axis.
\end{restatable}

\vspace{-14pt}

\proof{See \appref{app:proof_seam_factorisation} (by standard change-of-variables).}

% \vspace{4pt}
% \keypoint{Summary}
% Singular vector paths in $\gZ$ are\textit{ the} latent curves along which $g$ changes in the corresponding seam on $\gM_g$ (\Cref{lem:seams}). 
\Cref{thm:seam-factorisation-general} states that $p_\mu$ decomposes as a product of factors, which, under C1, are 1-D densities over seams, each the push–forward of the marginal over an \textit{axis-aligned} latent s.v.\ path. This satisfies two requirements of disentanglement under D\ref{def:disentanglement}, leaving only that factors must also be 
% \textit{axis-aligned} (in general they may curve, as in \Figref{fig:seams}) and that factors are 
\textit{independent}, i.e.\ only factor $i$ changes over seam $i$, which requires 
% \textbf{C1} and 
\textbf{C2}.

\label{sec:lin_to_stat}

\label{sec:non-linear}

\begin{restatable}[\textbf{Disentanglement $\bm{\Leftrightarrow}$ C1-C2}]{theorem}{thmnonlinear}
\label{thm:non_linear}
Let $g\cl:\gZ\!\to\!\gX$ be \emph{c.i.d.a.e.}\  and the prior factorise as $p(z)=\prod_{i=1}^d p_i(z_i)$.
% (e.g.\ standard Gaussian).
The push-forward density $p_\mu$ on the manifold $\gM_g$
% _g\eq \{g(z)\}$ 
is \emph{disentangled} (D\ref{def:disentanglement}) if and only if $g$ satisfies \textbf{C1} and \textbf{C2} \textit{a.e.}.
\end{restatable}
\vspace{-14pt}
\proof{See \appref{sec:non-linear-disentanglement}}

\thmref{thm:non_linear} means that \textbf{C1-C2} are \textit{precisely} the constraints needed for the canonical \underline{factorisation} (\Eqref{eq:pf-general-product}) to satisfy \Eqref{eq:disentanglement} and yield disentanglement: 
\textbf{C1} causes s.v.\ paths to \underline{\smash{axis-align}}; and
\textbf{C2} rules out factor $i$ varying in co-ordinates $z_{j}, j\cl\ne i$, ensuring seam factors are \underline{\smash{independent}} (see \Figref{fig:main_fig}).

\section{VAE + Diagonal Posteriors $\Rightarrow$ $\sim$C1-C2}
\label{sec:cleaning_up_covariances}
\vspace{-4pt}

Having defined disentanglement and characterised it in terms of conditions on the generative function, we now consider why it arises in  Gaussian VAEs with diagonal posteriors.
Prior work attributes VAE disentanglement to diagonal posteriors via {approximate} relationships \citep[][Eq.~11]{rolinek2019variational, kumar2020implicit}. 
This is actually \textit{exact} from the Price/Bonnet Theorem  \citep[e.g.][]{opper2009variational}. The ELBO with Gaussian posteriors is maximal when covariances satisfy 
\vspace{-2pt}
\begin{align}
    \label{eq:real_relationship}
    \!\!\Sigma_x\inv 
        \ &=\  
        \mI - \tfrac{1}{\beta}\E_{q(z|x)}[\tL_z(x)] 
        \nonumber\\  
        &\overset{*}{=}\ 
        \mI + \tfrac{1}{\beta\sigma^2}\E_{q(z|x)}[\mTm{\mJ_z}
            - (x \cl- d(z))^\top\tH_z]          
        \ ,
\end{align}

\vspace{-9pt}
where
 $\tL_{z}(x) \eq \nabla_z^2\log p_\theta(x|z)$ is the log likelihood Hessian;  
 $\mJ_{z} \cl\doteq \deriv{x}{z}$ and $\tH_z\cl\doteq \deriv{^2x}{z^2}$ are the Jacobian and Hessian of the decoder (all evaluated at $z\clin\gZ$); and (*) assumes a Gaussian VAE. 
It can be seen that \Eqref{eq:real_relationship} generalises the linear result ($\mM\inv\!$, \Eqref{eq:ppca-posterior}) and 
relates $\sigma^2 \cl\doteq \Var[x|z]$ and $\Sigma_x \cl\doteq \Var[z|x]$, showing that (un)certainty in $\rx$ and $\rz$ go hand in hand.
 
Notably, \Eqref{eq:real_relationship} implies that diagonal posteriors encourage decoder derivatives to diagonalise (in aggregate/expectation), and we know specific diagonal derivatives give disentanglement (\thmref{thm:non_linear}). To formalise the link, we consider
% of \Eqref{eq:real_relationship}
\vspace{-2pt}
\begin{restatable}{property}{assnoncancellation}
    \label{prp:non-cancellation}
    For fixed $x$ and $z$, the matrices $\mJ_z^\top\mJ_z$ and $(x-d(z))^\top\tH_z$ (as in \Eqref{eq:real_relationship})
    % are each diagonal.
    % For $z$ concentrated around $\E[z|x]$, $\mTm{\mJ_z}$ and $(x \cl- d(z))\!^\top\tH_z$  (terms in \Eqref{eq:real_relationship}) 
    are \underline{each} diagonal.
    % \footnote{
    %     Noting that these properties hold in expectation, we proceed as if they hold pointwise to simplify presentation.
    %     }
\end{restatable}
% \vspace{-2pt}

% \begin{itemize}[leftmargin=0.75cm]
%     \item[\textbf{C1)}] 
%     \vspace{-4pt}
%     \item[\textbf{C2)}] 
% \end{itemize}

\vspace{-8pt}
\begin{restatable}[\textbf{VAE Disentanglement}]{lemma}{lemdiagterms}
    \label{lem:diag-terms}
    % A Gaussian VAE with diagonal posteriors satisfying \textbf{P}\ref{prp:non-cancellation} induces \textbf{C}1 and \textbf{C}2.
    % 
    If a trained Gaussian VAE with diagonal posteriors induces \textbf{P}\ref{prp:non-cancellation} for $z$ in a neighbourhood of $e(x)$, then it induces \textbf{C}1 on that neighbourhood and \textbf{C}2 to a first order approximation.
    % \vspace{-8pt}
\end{restatable}

\vspace{-18pt}
\proof{See \appref{app:proof_decoder_constraints}.
    (C1 follows from SVD of $\mJ_z$; C2 follows since directions $r(z) \eq x-d(z)\clin\gX$ of the directed Hessian are, to a first order approximation, tangent to $\gM_d$.)
    % (C1 follows from the SVD of $\mJ_z$; C2 follows to first order from diagonality of the directed Hessian term together with the local expansion $x-d(z)=-\mJ_{e(x)}(z-e(x))+O(\|z-e(x)\|^2)$.)
}

Thus, \textbf{\Eqref{eq:real_relationship} suggests a mechanism for VAE disentanglement}:
diagonal posteriors encourage local termwise diagonality of the decoder derivatives on posterior-supported regions; this yields \textbf{C1} exactly and \textbf{C2} to first order, hence biases the push-forward density toward disentanglement. 

At the same time, qualifications in the chain linking \Eqref{eq:real_relationship} to pointwise disentanglement conditions C1-C2 (such as ``in aggregate/expectation'', restriction to posterior support, ``first approximation'') \textbf{may justify why VAE disentanglement is inconsistent in practice} \citep{locatello2019challenging}.

% Thus, 
%     to the extent a Gaussian VAE with diagonal posteriors diagonalises \textit{each} term in 
%     % Property P\ref{prp:non-cancellation} by 
%     \Eqref{eq:real_relationship} over a subset of latent space, the decoder's push-forward density is disentangled. 
    % 

    % \Eqref{eq:real_relationship} thus offers a plausible mechanism for how disentanglement emerges in VAEs. 

% \textbf{Why does $\bm{\beta}$ affect disentanglement?} 
It can be shown that $\beta$ implicitly controls the likelihood variance (see \appref{app:beta}), in particular \emph{higher $\beta$ widens posteriors}. Meanwhile \Eqref{eq:real_relationship} shows disentanglement conditions C1-C2 are encouraged over posteriors. 
% (while $\beta$-ELBO remains a valid objective) 
Thus, \textbf{\Eqref{eq:real_relationship} suggests why increasing $\beta$ enhances disentanglement} \citep{betavae, burgess2018understanding}: it broadens where disentanglement is encouraged and increases posterior overlap where simultaneous constraints apply (see Fig.~\ref{fig:taming}).
% (see Fig.~\ref{fig:taming}, \textit{right}).
    % } 

Lastly, we note that \thmref{thm:non_linear} suggests that where disentanglement is observed, constraints \textbf{C1-C2} should, to some extent, hold. Empirical evidence of this is reviewed in \Secref{sec:empirical}.
% 
% \paragraph{Remark (expectations and tangent-frame Hessian).}
% \Eqref{eq:real_relationship} is \emph{exact}; the ``in expectation'' qualifier refers to averaging over the encoder distribution (training region), not to a linearisation.
% Our structural analysis uses the \emph{directed} (tangent-frame) Hessian, aligned with the right/left singular vectors of the decoder Jacobian.
% Thus C2 constrains how singular values vary with their own latent coordinate; it does \emph{not} assume a diagonal Hessian in the ambient (pixel) basis as in additive/compositional settings.
% We discuss this formally in Appendix~A (\emph{Directed Hessian is Tangent to the Manifold}).

\section{Identifiability}
\label{sec:identifiability}
\vspace{-4pt}

We now investigate whether a model capable of fitting data generated under the model class, will learn the \textit{true} generative factors, up to a particular symmetry say, or if it might find a spurious factorisation.

\begin{restatable}[\textbf{LVAE Identifiability}]{theorem}{corlinear}
\label{cor:linear_ident}
Let data be generated under the linear Gaussian LVM (\Eqref{eq:ppca_assumps}) with ground-truth $g(z) \eq \mW z$, 
$\mW\!=\mU\mS\mV^\top\inRR{m}{d}$ of full column rank and \emph{distinct} singular values. 
Let an LVAE with \emph{diagonal} posteriors be trained on $n$ samples, and as $n\!\to\!\infty$ its learned parameters yield $p_\mu^{\text{(d)}} \equiv p_\mu^{\text{(g)}}$ on the mean manifold. 
Then the LVAE achieves disentanglement (D\ref{def:disentanglement}) and identifies ground-truth independent components on $\gM_{g}$ up to permutation and sign (\textbf{P\bm{\&}S}).
\end{restatable}
\vspace{-14pt}
\begin{proof}
    See \appref{app:proof_identifiability_LVAE} (Follows SVD uniqueness).
\end{proof}
\vspace{-2pt}
Thus, if an LVAE fits the data, latent axes identify the ground truth independent factors.
\vspace{4pt}

\begin{remark}[\textit{$\mV$ immaterial}]\label{rmk:v_immaterial}
    Ground-truth right singular vectors $\mV$ (of $\mW$) are not recoverable from $p(x)$ under the PPCA/LVAE model; this \emph{is not} a lack of identification. With a standard Gaussian prior, any orthonormal change of basis of $z$ preserves independence and leaves $p(x)$ unchanged.
    The only data-relevant object is $\mU\mS$;
    the arbitrary basis ($\mV$) in which $\mW$ was written
    % , an arbitrary reference frame, 
    has no bearing on $p(x)$.\footnote{This is analogous to the difference between a linear transformation and its matrix representation, requiring a specific basis.}
\end{remark}

% To generalise the linear result, we start by showing that if a manifold density admits a seam factorisation (as in \Eqref{eq:disentanglement}), that seam factorisation is \textit{intrinsic} to  $p_\mu$ and \textit{unique} (up to P\&S), agnostic to any parameterisation. It follows that if a push-forward model satisfying C1-C2 fits $p_\mu$, then its seams must, in general, align with the intrinsic seams of $p_\mu$ (up to P\&S). Hence a Gaussian VAE with diagonal posteriors fitting $p_\mu$ (approximately) \textit{identifies ground truth factors} (up to P\&S). 

To generalise the linear result, we first isolate the part that depends only on the manifold density itself: if $p_\mu$ admits a seam factorisation in local orthonormal seam coordinates, then those seams are \textit{intrinsic} to $p_\mu$ and unique (up to P\&S).
We then show that any matching pushforward satisfying C1–C2 necessarily recovers those seams. (All proofs in Appendix~\ref{app:intrinsic-uniqueness})

\begin{restatable}[\textbf{Seams are Intrinsic}]{lemma}{lemintrinsicseams}\label{thm:intrinsic-uniqueness}
    Let $\gM\subseteq\gX$ carry a manifold density $p_\mu$. 
    Assume that on a regular set $\gM_{\mathrm{reg}}$ there exist 
    local coordinates $u=(u_1,\ldots,u_d)$, whose coordinate lines are 1-D seams and directions form an orthonormal basis of $T_x\gM$ at each $x\in\gM_{\mathrm{reg}}$, together with 1-D densities $\{f_i\}$ such that
    % scalar functions $\{u_i(x)\}_{i=1}^d$ (each varying only along a 1-D curve through $x$, i.e.\ a \emph{seam}) and 1‑D densities $\{f_i\}$ such that
    % \vspace{-4pt}
    \begin{equation}\label{eq:intrinsic_factorisation}
        p_\mu(x) \;=\; \prod_{i=1}^d f_i\!\big(u_i(x)\big), 
        \qquad x\in\gM_{\mathrm{reg}}.
    % \vspace{-4pt}
    \end{equation}
    Let $\tH_x \cl\doteq \nabla^2_{\gM}\log p_\mu(x)$ denote the intrinsic Hessian on the manifold, and assume its eigenvalues are pairwise distinct \textit{a.e.}.
    % \ on $\gM_{\mathrm{reg}}$.
    % 
    Then for each $x\clin\gM_{\mathrm{reg}}$, the $d$ seam directions 
    % (along which exactly one $u_i$ varies)  
    are determined intrinsically by $p_\mu$, as eigenvectors of $\tH_x$, unique up to P\&S.
\end{restatable}

\begin{restatable}[\textbf{A matching pushforward finds seams}]{lemma}{lemdecoderfindsseams}\label{cor:intrinsic-uniqueness}
    Under assumptions of \Cref{thm:intrinsic-uniqueness}, let $d:\gZ\!\to\!\gX$ be \emph{c.i.d.a.e.} with factorised prior $p(z)=\prod_i p_i(z_i)$ and push‑forward density $p_\mu^{(d)}\equiv p_\mu$ \emph{matching} on $\gM_d\doteq\{d(z)\} = \gM$. 
    If $d$ satisfies \textbf{C1–C2} \textit{a.e.}, then for any $z$ and $x=d(z)$:
    % 
    % \vspace{-2pt}
    \begin{itemize}[leftmargin=0.5cm,itemsep=-2pt,topsep=-2pt]
        \item left singular vectors $\mU_z$ of $\mJ_z$ coincide with the intrinsic seam directions in \Cref{thm:intrinsic-uniqueness} (up to P\&S);
        \vspace{-2pt}
        \item the images under $d$ of singular-vector paths (D\ref{def:sv-path}) are exactly the intrinsic seams through $x$;
        \vspace{-2pt}
        \item along the $i$‑th intrinsic seam, the factor $f_i$ is the 1‑D push‑forward of $p_i(z_i)$ (as in \Cref{thm:seam-factorisation-general}).
    \end{itemize}
\end{restatable}
\vspace{-6pt}

\begin{restatable}[\textbf{Gaussian VAE Identifiability}]{theorem}{cornonlinear}
\label{cor:nonlinear-ident}
Let data be generated by \emph{c.i.d.a.e.}\ $g:\gZ\!\to\!\gX$  with standard Gaussian prior.
% factorised prior $p(z)\eq\prod_{i=1}^d p_i(z_i)$. 
Let a Gaussian VAE with \emph{diagonal} posteriors learn a decoder $d\cl:\gZ\!\to\!\gX$. Suppose both $g$ and $d$ satisfy \textbf{C1–C2} 
and 
manifold densities match: 
$p_\mu^{(d)}\!\!\equiv\! p_\mu^{(g)}$ on $\gM\eq\{g(z)\}\eq\{d(z)\}$. If the eigenvalues of the tangent Hessian (see proof) are pairwise distinct \textit{a.e.}, then $d$ identifies ground‑truth independent components on $\gM_g$, up to P\&S.
\end{restatable}

\vspace{-8pt}

Thus, if a Gaussian VAE with diagonal covariances fits the push-forward of a Gaussian distribution under the conditions of \Cref{thm:non_linear}, then it identifies and disentangles the ground truth generative factors (up to permutation/sign).

\vspace{4pt}
\begin{remark}
    P\&S symmetry is optimal since seams follow singular vectors $\vu_i$ that have no inherent order/orientation. 
\end{remark}
\vspace{-2pt}

We can also consider fitting a Gaussian VAE to data sampled from the push-forward of other priors.

\vspace{-4pt}
\begin{corollary}[\textbf{BSS}]
    In \Cref{cor:nonlinear-ident}, if priors $p^{(g)}(z)$ and $p^{(d)}(z)$ factorise
    and $p^{(g)}_\mu\cl\equiv p^{(d)}_\mu$ with C1-C2 holding \emph{a.e.}, then the seam decomposition $p_\mu$ on $\gM$ is unique, up to P\&S, 
    and $g$ and $p^{(g)}(z)$ are recoverable up to an axis‑aligned diffeomorphism $\phi\doteq g\inv \!\circ d:\gZ\to\gZ$.
\end{corollary}
\vspace{-16pt}
\begin{proof}
    Immediate from first part of \Cref{cor:nonlinear-ident} proof, which does not depend on the form of $p(z)$. The diffeomorphism follows since $\gM_g\eq\gM_d$ and by injectivity of $d,\,g$.
\end{proof}
\vspace{-9pt}

\vspace{4pt}
\begin{remark}[\textit{ICA \emph{un}identifiability}]
    Classical non-linear ICA aims to identify ground truth factors of data generated by an LVM (\Eqref{eq:lvm}), 
    % often with $p_\theta(x|z) \eq \delta_{x-d(z)}$). 
    which is impossible in general for Gaussian $p(z)$, without extra side information, e.g.\ auxiliary 
    % or temporal non‑stationarity
    variables \citep{khemakhem2020variational, locatello2019challenging}.
    Our results show that identifiability of independent components of $p(x)$ by any pushforward model (i.e.\ disentanglement) requires (a) that the manifold density $p_\mu$ factorises over the manifold (intrinsic seams); and (b) the pushforward function satisfies C1–C2.\footnote{
        The ICA unidentifiability proofs we are aware of rely on arbitrary rotation of the Gaussian prior (\textit{cf} $\mR$ in \Eqref{eq:ppca-posterior}), which C1 prohibits, as readily seen in the linear case (\Figref{fig:linear_results}). (See also \Cref{rmk:v_immaterial}.)} 
    Even with a rotationally \textit{asymmetric} factorised $p(x)$, C1–C2 remain necessary for each latent co-ordinate to govern an independent factor of the pushforward density. Modelling  posteriors and constraining them to be diagonal indirectly (approximately) imposes C1-C2, thus a VAE (approximately) identifies independent components without additional side information.
\end{remark}

In summary, we have defined disentanglement, shown that it holds if and only if C1-C2 hold, and that seam factors are unique and therefore identifiable, up to natural symmetry.

\begin{figure}[!b]
    \centering
    % \vspace{-12pt}
    \captionsetup[subfigure]{labelformat=empty,skip=0.05\baselineskip}
    \subfloat{  
    \scalebox{0.7}{ 
        \begin{tikzpicture}
    \begin{axis}[
        width=0.59\textwidth,
        height=5.1cm,
        ybar,
        bar width=15pt,
        ylabel={\footnotesize Distance to Orthogonality},
        symbolic x coords={dSprites, syn-lin, syn-nonlin, MNIST, fMNIST},
        xtick=data,
        xtick style={draw=none},
        xticklabel style={font=\footnotesize},
        % axis x line=bottom,
        % axis y line=left,
        % nodes near coords,
        % nodes near coords align={vertical},
        ymin=-0.01,
        ymax=2.1,
        legend style={at={(0.4,0.95)}, anchor=north, legend columns=1, font=\footnotesize},
    ]
    % Values from Table 2 in the paper
    \addplot coordinates {(dSprites, 0.76) (syn-lin, 0.00) (syn-nonlin, 0.18) (MNIST, 1.59) (fMNIST, 1.36)};
    \addplot coordinates {(dSprites, 1.73) (syn-lin, 0.34) (syn-nonlin, 0.55) (MNIST, 1.93) (fMNIST, 2.02)};
    % \addplot coordinates {(MNIST, 0.81) (Fashion-MNIST, 0.77) (CIFAR-10, 0.73) (CelebA, 0.85)};
    \legend{Diag $\Sigma_x$, Full $\Sigma_x$}
    \end{axis}
\end{tikzpicture} } 
    }%
    \hfill
    \caption{
    % (\textbf{\textit{left}}) 
        \textbf{
            Empirical support for P\ref{prp:non-cancellation}}:
        % results of 
        VAEs with diagonal $\Sigma_x$ show higher Jacobian column-orthogonality \citep{rolinek2019variational}.
    }
    \label{fig:p1_support}
\end{figure}

\section{Empirical Support}
\vspace{-4pt}
\label{sec:empirical}

\begin{figure*}[t!]
    % \vspace{-8pt}
    \centering
    \begin{minipage}{0.8\textwidth}
        \captionsetup[subfigure]{labelformat=empty,skip=0.05\baselineskip}
        \subfloat[Diagonal posteriors]{\includegraphics[width=0.48\linewidth]{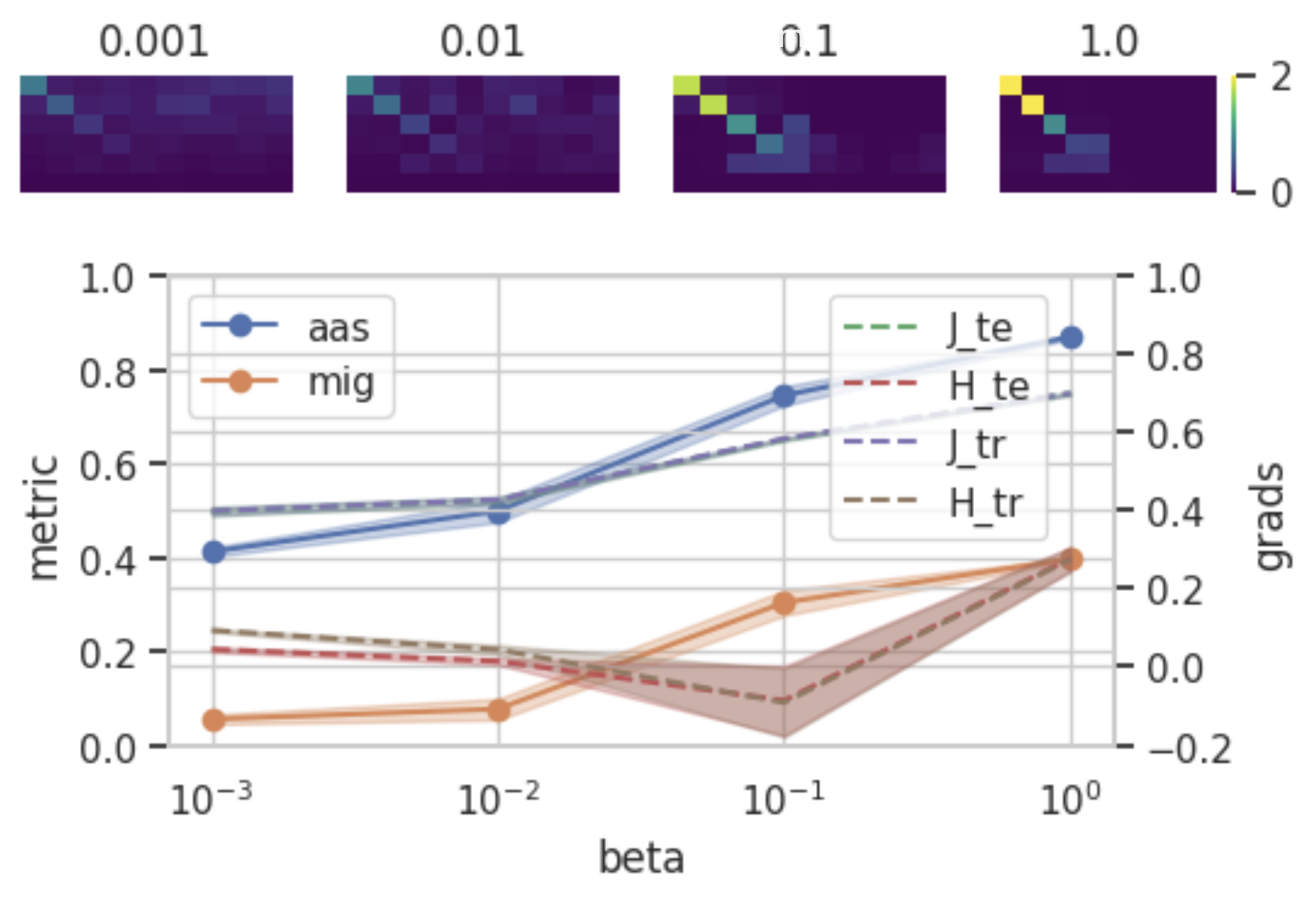}}%
        \hfill
        \subfloat[Full posteriors]{\includegraphics[width=0.48\linewidth]{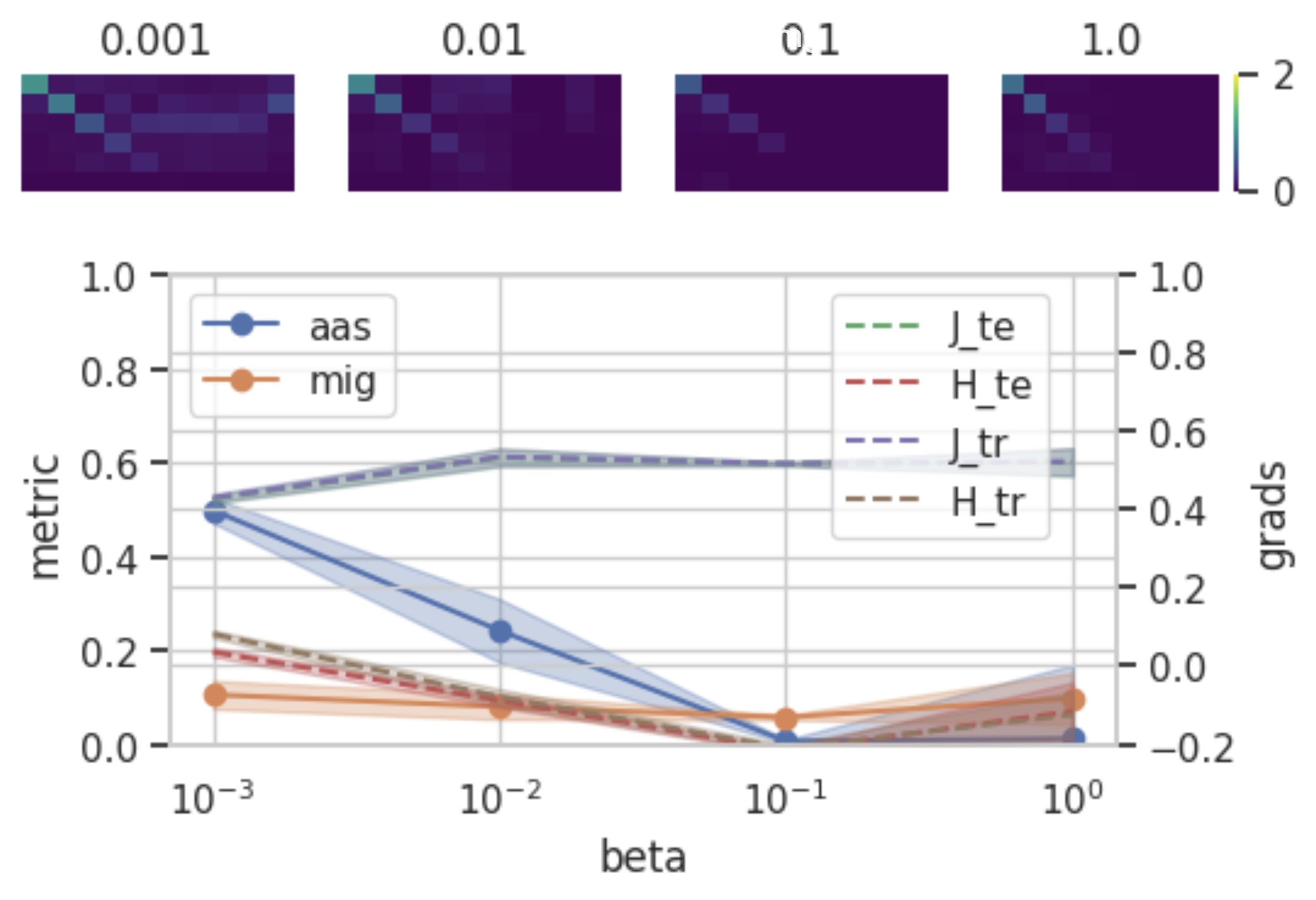}}%

    \end{minipage}
    \hfill
    \begin{minipage}{0.19\textwidth}
        \centering
        % \vspace{-100}
        % \vspace{-30pt}
        \captionsetup[subfigure]{labelformat=empty,skip=0.25\baselineskip}
        \subfloat[\scriptsize \centering \,Diag\newline $\beta\!: 1\quad$]{\includegraphics[width=0.33\linewidth]{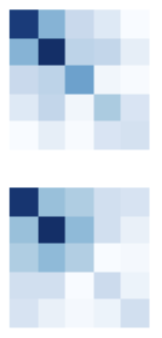}}%
        \subfloat[\scriptsize \centering \,Diag  \newline $1e\sminus3$]{\includegraphics[width=0.33\linewidth]{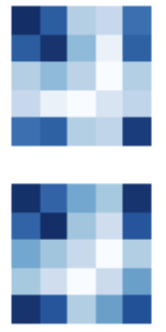}}%
        \subfloat[\scriptsize \centering \,Full\newline \!$1$]{\includegraphics[width=0.33\linewidth]{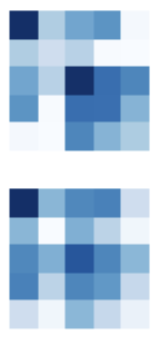}}%
        \vspace{4pt}
        \subcaption{\footnotesize 
            \centering (\textit{t}) $\smaller \mJ^\top_z\mJ$;  \newline
            (\textit{b}) $\smaller  (x\sminus d(z))\tH_z$ 
            } 

    \end{minipage}
    \vspace{-6pt}
    \caption{
         \textbf{Diagonal vs Full Posteriors}: (\textbf{\textit{left}}) 
         (\textit{bottom}) 
            disentanglement metrics and estimated diagonality of \Eqref{eq:real_relationship} terms (see \appref{app:metrics}). With diagonal posteriors, disentanglement and diagonality are correlated (supporting P\ref{prp:non-cancellation}), relative to full posteriors; 
         (\textbf{\textit{top}}) heatmaps of mutual information between model latents and ground truth factors. 
         (\textbf{\textit{right}}) derivatives in \Eqref{eq:real_relationship}. terms are less diagonal for lower beta or full posteriors. 
         (All results averaged over multiple runs)
         }
    \label{fig:non_linear_results}
    \vspace{-12pt}
\end{figure*}

We include empirical results to illustrate disentanglement and support our claims.
From our analysis, we expect 
    (1) diagonalised Jacobian and Hessian terms to correlate with disentanglement generally (\thmref{thm:non_linear}); and
    (2) VAEs with diagonal posteriors to promote diagonalised derivatives, likely more so in the Jacobian term than the Hessian term (\Secref{sec:cleaning_up_covariances}).

Both (1) and (2) are clearly illustrated in the linear case where ground truth factors are  known analytically.
\Figref{fig:linear_results} shows results for diagonal and full covariance LVAEs learning  Gaussian parameters (see caption for details). 
    (\textit{left}) Both models learn optimal parameters as expected.
    (\textit{centre}) However, only diagonal covariances cause right singular vectors $\mV$ of $\mJ_z$ to converge to standard basis vectors, i.e.\ $\mV\cl\to\mI$, inducing \textbf{C1} (prediction 2). 
    (\textit{right}) Thus latent traversals map to independent components along left singular vectors $\vu^i$, yielding disentanglement (D\ref{def:disentanglement}) (prediction 1). This provides compelling evidence that diagonal covariances ``break the  rotational symmetry'' of a Gaussian prior. 
    
    Though simplified, the linear case is fundamental since, from our analysis, disentanglement under general pushforward functions follows a similar rationale. Interestingly, \Figref{fig:linear_results} shows that learning parameters of a disentangled model is notably slower (\textit{left}), due to the rate $\mV\cl\to\mI$ (\textit{centre}). A diagonal covariance model must find one of a finite subset of the infinite solutions set of a full covariance model.

Various studies show empirical support for (1) and (2). \citet{rolinek2019variational} show that columns of a decoder's Jacobian are more orthogonal (i.e.\ $\mV\cl\to\mI$, \textbf{C1}) in  VAEs with diagonal posteriors than those with full posteriors (\Figref{fig:p1_support}), and that diagonality correlates with disentanglement. Supporting (1), \citet{kumar2020implicit} show that \textit{directly} inducing column-orthogonality in the decoder Jacobian promotes disentanglement. 

Complementary evidence comes from GANs, where the same decoder-side picture appears without any posterior model. 
\citet{ramesh2018spectral} trace locally disentangled directions by following Jacobian singular vectors of a GAN generator. 
\citet{wei2021orthogonal} directly regularise orthogonality of Jacobian-induced output variations across latent co-ordinates, closely aligned with diagonalising $J_z^\top J_z$, hence with C1. 
\citet{peebles2020hessian} directly penalise mixed second derivatives of the generator, a Hessian-side regulariser closely related to the mechanism in \Eqref{eq:real_relationship}, hence to the emergence of C2. 
Together these results support the view that disentanglement is controlled by derivative structure of a general pushforward map, not specific to the VAE objective.

% Indeed, the $\mJ_z^\top\mJ_z$ term \textit{alone}  
% is found to be approximately diagonal (\Figref{fig:p1_support}) \citep{rolinek2019variational, kumar2020implicit}, suggesting that \textit{both} terms diagonalise. 

\textbf{dSprites}:
Further to prior evidence, we train diagonal and full posterior Gaussian VAEs ($d\eq10$) on the \textit{dSprites} dataset with 5 known generative factors (results averaged over 5 runs). How well each latent co-ordinate identifies a ground truth factor can be estimated explicitly from their \textit{mutual information}, and a function of mutual information is often used to evaluate a model's overall disentanglement, e.g.\ the \textit{mutual information gap} (\text{MIG}, \citet{tcvae}). 

\Figref{fig:non_linear_results} (\textit{main plots}) reports MIG, \textit{axis alignment score} (AAS) (a novel metric based on entropy of the mutual information \textit{distribution}) and derivative diagonality estimates plotted against $\beta$ (see \appref{app:metrics} for details).
For diagonal posterior VAEs (\textit{left}), disentanglement and diagonality both broadly increase with $\beta$. For full posteriors (\textit{right}) no clear trend is observed. The heatmaps (top, aligned with main plots by $\beta$) show mutual information between each latent co-ordinate and ground truth factor (ordered greedily to put highest scores along the diagonal). We see that for diagonal posterior VAEs, disentanglement increases with $\beta$ and that individual latent co-ordinates (horizontal) correlate with distinct ground truth factors (vertical), whereas that trend is not observed for full posteriors. 
For illustration, \Figref{fig:non_linear_results} (right) shows heatmaps of the $d\cl\times d$ derivative terms in \Eqref{eq:real_relationship} ($\mJ_z^\top\mJ_z$ term top, Hessian term bottom), each for 
    diagonal covariances, $\beta\eq1$ (\textit{l}); 
    diagonal covariances, $\beta\eq0.001$ (\textit{c}); and 
    full covariances, $\beta\eq1$ (\textit{r}) (each averaged over a batch).
    
We see that diagonalisation of the Jacobian term is more evident than that of the Hessian, consistent with our analysis of \Eqref{eq:real_relationship} vs disentanglement conditions C1-C2 (\Secref{sec:cleaning_up_covariances}).

\textbf{CelebA}:
We include comparable analysis on the CelebA dataset of natural face images in Appendix~\ref{app:celeba}.

\textbf{Varying $\bm\beta$}:
Our understanding of the interplay between $\beta$ and disentanglement is that 
    higher $\beta$ implies higher expected noise, ``blurring'' reconstructions, but enhancing disentanglement; 
    % via widened posteriors; 
    while lower $\beta$ (less assumed noise) tightens reconstructions but limits disentanglement to concentrated, potentially disconnected regions of $\gZ$ (\Secref{sec:cleaning_up_covariances}). 
This suggests that the heuristic of starting with high $\beta$ and reducing it during training may give both disentangled and higher quality samples.
We run experiments on the dSprites dataset and report results in \appref{app:profile_results} for constant $\beta$ baselines (1, 0,001), and exponentially reducing $\beta$ ($1\cl\to0.001$) over training. As predicted, the results show that annealing $\beta$ gives both sharp reconstructions and good disentanglement. We note the resemblance of this to ``de-noising'' in autoencoders and diffusion models, suggesting that dynamically varying $\beta$ is an interesting, potentially principled, direction for future research.

\section{Related Work}
\vspace{-6pt}

% Many works study VAEs, or disentanglement in other paradigms. We focus on those investigating disentanglement in VAEs. 

\citet{betavae} first showed that disentanglement in VAEs is enhanced by setting $\beta\cl>1$ in the ELBO (\Eqref{eq:elbo}). 
\citet{burgess2018understanding} conjectured that diagonal posterior covariances may cause disentanglement. 
\citet{rolinek2019variational} showed supporting empirical evidence (Fig.~\ref{fig:seams}) and derived an approximate relationship between diagonal posteriors and Jacobian orthogonality, conjectured to cause disentanglement. \citet{kumar2020implicit} generalised the argument, reaching an approximation to the identity in \Eqref{eq:real_relationship}. We make the link between posterior covariances and decoder derivatives precise  in \Eqref{eq:real_relationship} and, by giving disentanglement a formal definition, show how it follows from \Eqref{eq:real_relationship} via constraints C1-C2, resolving the conjecture.

\citet{lucas2019don, bao2020regularized} and \cite{koehler2021variational} study properties of linear VAEs. Notably \citet{lucas2019don} show the equivalence of $\beta$ and Var$[x|z]$ in Gaussian VAEs (which we generalise in \appref{app:beta}); and prove identifiability of LVAEs (which we generalise to the non-linear case). 
\citet{zietlow2021demystifying} show that disentanglement can be sensitive to perturbing the data. 
\citet{reizinger2022embrace} seek to relate the ELBO to \textit{independent mechanism analysis} \citep{gresele2021independent}, which encourages column-orthogonality in the mixing function of ICA.\footnote{
    We report discrepancies in \citet{reizinger2022embrace} in \appref{app:errors}.}
% We show that Jacobian orthogonality (\textbf{C1}) is insufficient for disentanglement/identification of independent components, which also requires (\textbf{C2}). 
Disentanglement relates closely to (noisy) ICA \citep[\Secref{sec:bg}, e.g.][]{hyvarinen1998noisy}. ICA assumes the same generative LVM but does not model the posterior, which we show is critical to disentanglement in VAEs, or otherwise explicitly induce C1-C2, necessary for disentanglement, rather asymmetric priors or side information are typically required.

Our pushforward perspective unifies several GAN disentanglement results. 
    \citet{ramesh2018spectral} identify disentangled directions from singular vectors of the generator Jacobian, close to our seam/s.v.\ path view.
    \citet{wei2021orthogonal} encourage output variations induced by different latent co-ordinates to be orthogonal, close to our C1. 
    \citet{peebles2020hessian} suppress mixed second derivatives of the generator, related to the mechanism behind C2. 
    We show, for a general pushforward density, why such derivative-side regularisers are not merely heuristic, but target structure necessary and sufficient for disentanglement in smooth pushforward models.

\citet{chadebec2022geometric} and \citet{arvanitidis2018latent} consider paths in latent space defined by the inverse image of paths over the data manifold (our \textit{s.v.\ paths}). \citet{pan2023geometric} claim that the data manifold is identifiable from a geometric perspective assuming Jacobian-orthogonality, differing to our probabilistic factorisation view.
\citet{bhowalvariational} consider linear and non-linear components of the encoder/decoder, loosely resembling our Jacobian SVD view. However, dissecting a function into linear/non-linear components is not well defined, whereas the SVD is unique (up to permutation/sign). 
\citet{buchholz2022function, buchholz2025robustness} analyse identifiability of function classes, e.g.\ proving that \textit{conformal maps} are identifiable and \textit{orthogonal coordinate transformations} (satisfying C1) are not. Our theoretically derived disentanglement constraints C1-C2 define a function class corresponding to a density factorisation that is unique and hence identifiable, up to symmetry (\thmref{cor:nonlinear-ident}). \citet{brady2023provably} and \citet{lachapelle2023additive} analyze \textit{additive/compositional} decoders with block‑diagonal Hessians in pixel space, strictly stronger than our \thmref{thm:non_linear}, which allows, e.g., rotations/scale/colour changes that would not be block‑diagonal in pixel space.

\section{Conclusion}
\vspace{-4pt}

Unsupervised disentanglement of generative factors of the data is of fundamental interest in machine learning, whether for generating synthetic data, ``breaking down'' data into its fundamental constituents or identifying specific latent attributes of interest. Thus, having a rigorous definition of disentanglement and understanding how it arises \textit{for free} in VAEs may be useful across machine learning paradigms. 

We present a definition of disentanglement (D\ref{def:disentanglement}) grounded in statistical independence, formalising the notion that each latent variable changes while leaving the others the same. Disentanglement amounts to factorising the manifold density into independent components, each factor depending on a single latent co-ordinate. 
We show that $\beta$ of a $\beta$-VAE controls the likelihood variance, explaining the empirical observation that $\beta\cl>1$ promotes disentanglement while degrading generative quality (and $\beta\cl<1$ mitigates \textit{posterior collapse}).

Our main results (\Cref{def:disentanglement}, \Cref{thm:seam-factorisation-general}, \Cref{thm:non_linear}, \Cref{thm:intrinsic-uniqueness}, \Cref{cor:intrinsic-uniqueness}, \Cref{cor:nonlinear-ident}) are \textbf{not specific to VAEs}, but are general to smooth push-forwards of a factorised prior, as also in GANs and flows. Indeed, several prior works show that generative factors can be identified in the GAN latent space via generator derivatives \cite{ramesh2018spectral, wei2021orthogonal, peebles2020hessian}.
Linking these results to the VAE disentanglement literature, we show that under specific conditions on the decoder Jacobian (C1-C2), the push-forward density over the data manifold factorises as a product of independent 1-D pushforwards. Indeed, the factorisation $p(z)\eq\prod_ip(z_i)$ in latent space can be considered \textit{projected/lifted} onto the manifold, with independent densities over \textit{singular vector paths} in latent space pushed-forward to independent densities over manifold \textit{seams}. We show that the C1/C2 constraints needed for disentanglement are encouraged in aggregate / in expectation by a VAE with diagonal posteriors, justifying both why disentanglement arises and why it is unreliable \citep{locatello2019challenging}.
Furthermore, independent factors are \textit{identifiable}, which is  significant in light of the \textit{unidentifability} of ICA with Gaussian priors.

VAEs and their variants form part of many state-of-the-art modelling pipelines, e.g.\ latent diffusion \citep[e.g.][]{stablediff, pandeydiffusevae, yang2023disdiff, zhang2022unsupervised} and LLMs. Other recent works show that supervised learning \citep{dhuliawala2023variational} and self-supervised learning \citep{bizeul2024probabilistic} can be viewed as latent models trained under ELBO variants. In future work we will look to extend our results to these other learning paradigms.

Neural network models are often considered too complex to explain, which is concerning given their increasing adoption in society. An improved theoretical understanding seems essential to optimally and safely take full advantage of their progress, particularly in critical systems. We hope our work is a useful step in that direction, providing new insight into how a data density can decompose over independent generative factors.
% Interestingly, our analysis show how a prior ``pushes forward'' can be considered abstractly from the potentially highly complex non-linear nature of the decoder.

% Carl is grateful to Gabriel Peyr\'e, Maarten de Hoop and Ivana Balazevic for helpful discussions over the course of this work; and to CFM for funding. This work was performed using HPC resources from GENCI–IDRIS (Grant 2025-AD011015749R1).

\clearpage
\bibliography{example_paper}
\bibliographystyle{icml2026}

\clearpage
\appendix

\section{Singular Vector Paths \& Seams}\label{app:defs}

\begin{restatable}[\textbf{$\bm{i}$-th singular–vector path}]{definition}{defsing}\label{def:sv-path}
    % \begin{definition}[Singular–vector path]    
    Let $g\cl:\gZ\to\gX$ be \emph{c.i.d.a.e.}. For  $z^*\clin\gZ_{\mathrm{reg}}$, $i\clin\{1,\ldots,d\}$,  
    the $i$-th \emph{singular–vector path} (\textbf{s.v.\ path}) through $z^*$ is any $C^1$ curve $t\mapsto z^i_t$ with $z^i_0=z^*$ satisfying
    \[
    \tfrac{d}{dt} z^i_t \;=\; \vv^{i}(z^i_t)
    \qquad\text{for } t \text{ in its maximal interval  $I_{z^*,i} \subseteq\R$\,.}
    \]
    We denote the path set by $\gV^i_{z^*}\doteq\{ z^i_t:\; t\in I_{z^*,i}\}\subseteq \gZ_{\mathrm{reg}}$.\footnote{Different choices of sign for $\vv^{i}$ reverse the time direction ($t\mapsto -t$) but generate the same path set.} (See \Figref{fig:seams}, \textit{left}, dash blue lines).
    % \end{definition}
\end{restatable}

\begin{restatable}[\textbf{$\bm{i}$-th seam}]{definition}{defseam}\label{def:seams}
    Let $g\cl:\gZ\to\gX$ be \emph{c.i.d.a.e.} with manifold $\gM_g=\{g(z)\}$.  
    For $z^*\clin\gZ_{\mathrm{reg}}$, $i\clin\{1,\ldots,d\}$, the $i$-th \textbf{seam} through $g(z^*)$ is any $C^1$ curve $t\mapsto x^i_t$ in $\gM_g$ with $x^i_0=g(z^*)$ satisfying
    \vspace{-4pt}
    \[
    \tfrac{d}{dt} x^i_t \;=\; s^{i}\!\big(g^{-1}(x^i_t)\big)\,\vu^{i}\!\big(g^{-1}(x^i_t)\big)
    \quad\text{for $t$ in $I_{z^*,i}$\,.
    }
    \]
    We denote the path set $\gM^{i}_{g,z^*} \doteq\{ x^i_t:\; t\in I_{z^*,i}\}\subseteq\gM_g$ and define \emph{seam coordinate}
    \vspace{-2pt}
    % \[
    \begin{align}
        u_i(t) &\;\doteq\; \int_{0}^{t} s^{i}\!\big(g^{-1}(x^i_\tau)\big)\, d\tau,
        \nonumber\\
        \text{so}\quad \tfrac{d}{dt}u_i(t) &\;=\; s^{i}\big(g^{-1}(x^i_t)\big),\ \ u_i(0)=0.
    \end{align}
    % \vspace{-2pt}
    % \]
    $u_i$ measures position along the seam in units of $s^{i}$  (strictly monotone as $s^{i}\cl>0$). (See \Figref{fig:seams}, \textit{right}).
\end{restatable}

\begin{restatable}[\textbf{Paths $\mapsto$ seams}]{lemma}{lempath}\label{lem:seams}
    Let $g\cl:\gZ\to\gX$ be \emph{c.i.d.a.e.}, $z^*\clin\gZ_{\mathrm{reg}}$, $i\clin\{1,\ldots,d\}$, and let $\gV^i_{z^*}$ be the $i$-th s.v.\ path through $z^*$.
    % (Def.~\ref{def:sv-path}). 
    Then the image of $\gV^i_{z^*}$ under $g$ is the $i$-th seam through $g(z^*)$:
    % (Def.~\ref{def:seams}), 
    % i.e.
    % \[
    $\gM^{i}_{g,z^*} \eq \{g(z) : z\in \gV^i_{z^*}\}$.
    % \]
\end{restatable}
\vspace{-12pt}
\begin{proof}
For $x^i_t\cl\doteq g(z^i_t)$, by the chain rule and SVD:
$
\tfrac{dx^i_t}{dt}
= \mJ_{z^i_t}\tfrac{dz^i_t}{dt}
= \mJ_{z^i_t}\vv^{i}(z^i_t)
= s^{i}(z^i_t)\,\vu^{i}(z^i_t)
$,
so $x^i_t$ satisfies Def.~\ref{def:seams}.
\end{proof}

\thmseamfactorisation*

\label{app:proof_seam_factorisation}
\begin{proof}
By a standard change–of–variables (on embedded manifolds) and $|\mJ_z^\top\mJ_z|=\prod_{i=1}^d s^{i}(z)^2$,
\[
p_\mu\!\big(g(z)\big)
\;=\;
\det(\mJ_z^\top \mJ_z)^{\mathsmaller{-}^1\!/\!_2}\,p(z)
\;=\;
\tfrac{\prod_{i=1}^d p_i(z_i)}{\prod_{i=1}^d s^{i}(z)}\,,
% \tag*{\text{\quad\ }}
\]
yielding \Eqref{eq:pf-general-product}.
Under C1, s.v.-paths are axis aligned in latent space and map to seams. For each $i$ and $x=g(z)$, the change–of–variables formula along $\gM^{i}_{g,z}$ ($i$-th seam through $x$ with coordinate $u_i$, Def.~\ref{def:seams})  at $t=0$ 
 gives the local pushed 1-D \textit{seam–density}
\begin{equation}
\label{eq:seam-1d-no-inverse}
f^{(z)}_i\!\big(u_i(t)\big)
\;\doteq\;
\tfrac{p_i\!([z^{i}_t]_i)}{\,s^{i}\!(z^{i}_t)}\,,
\qquad
\tfrac{d}{dt}u_i(t)=s^{i}\!(z^{i}_t),\ \ u_i(0)=0,
\end{equation}
where $t\mapsto z^{i}_t$ is the $i$-th singular–vector path through $z$ (Def.~\ref{def:sv-path}) with co-ordinate $z^{i}_i(t)$. 
Evaluating at $t=0$: 
$f^{(z)}_i\!\big(u_i(0)\big) = \tfrac{p_i(z_i)}{s^{i}(z)}$,
shows that seam-densities are the factors of \Eqref{eq:pf-general-product}.
\end{proof}

% \clearpage
\section{Disentanglement}

% \subsection{Proof of Non-linear Disentanglement}
\label{sec:non-linear-disentanglement}

\thmnonlinear*

\begin{proof}
    \!(C1/2\,$\Rightarrow$D1) By Thm.~\ref{thm:seam-factorisation-general}, the manifold density factorises pointwise as
    $p_\mu(g(z)) \eq \prod_{i=1}^d \!\tfrac{p_i(z_i)}{s^{i}(z)}$.
    By \textbf{C1}, $\mV_z\eq\mI$ for all $z$, so the $i$-th singular–vector path through $z$ is exactly the axis–aligned line $\{z': [z']_i\ \text{varies},\, [z']_{\neg i}\eq[z]_{\neg i}\}$; by \Cref{lem:seams} its image is the $i$-th seam through $x\eq g(z)$ following $\vu^{i}$.
    By \textbf{C2}, $s^{i}(z)$ depends only on $z_i$. Define the seam coordinate $u_i$ along the $i$-th seam as in D\ref{def:seams}; then $u_i$ is a strictly monotone function of $z_i$, hence the 1-D push–forward of $p_i$ along that seam is
    $f_i(u_i)=\big|\tfrac{du_i}{dz_i}\big|\inv p_i(z_i)=\tfrac{p_i(z_i)}{s^{i}(z_i)}$,
    Thus $p_\mu(g(z))=\prod_i f_i(u_i(z))$ with each $f_i$ evaluated on the $i$-th seam.
    Finally, since $u_i$ is monotone in $z_i$ and $\{z_i\}$ are independent, the random variables $\{u_i\}$ are independent; hence factors $\{f_i(u_i)\}$ are \textit{statistically independent} as required by D\ref{def:disentanglement}.

    (D1$\Rightarrow$\,C1/2) Assume $p_\mu$ is disentangled under $g$. By D\ref{def:disentanglement}, each factor $f_i$ is obtained by pushing forward $p(z_i)$ over an axis-aligned line in direction $\vz^i$, and that line's image follows $\mJ_z\vz^i \eq \mJ_z^i$ at $g(z)$  (column $i$ of $\mJ_z$) with 1-D density $f_i \eq |\mJ_z^i|\inv p(z_i)$.
    Thus, 
        $p(g(z)) \eq |\mJ_z|\inv p(z)$ 
    by standard change-of-variables but also 
        $p(g(z)) \eq \prod_i f_i = \prod_i |\mJ_z^i|\inv p(z_i)$, 
    by factor independence. 
    Hence $|\mJ_z| = \prod_i |\mJ_z^i|$, which occurs if and only if columns $\mJ_z^i$ are orthogonal,
    % only $f_i$ can change along the $i$-th seam, hence all other factors must be orthogonal to $\mJ_i$, 
    i.e.\ $\mJ_z^\top\mJ_z$ is diagonal or in SVD terms $\mJ_z \eq \mU_z\mS_z$ (\textbf{C1}). 
    
    Since $s_i\cl\doteq[\mS_{z}]_{ii}\eq|\mJ_z^i|$ then $f_i\eq p(z_i)/s_i$. Each factor $f_i$ is a function of $z_i$, so by independence, for $i\cl\ne j$, 
        $0 \eq \pderiv{\log f_i}{z_j} 
            \eq \pderiv{\log p_i}{z_j} - \pderiv{\log s_i}{z_j} 
            \eq - \pderiv{\log s_i}{z_j}$ (since $z_i$ are independent), 
    implying $\pderiv{s_i}{z_j}\eq 0$ for $i\cl\ne j$ (\textbf{C2}).
\end{proof}

\clearpage
\section{Proof of Decoder Derivative Constraints}
\label{app:proof_decoder_constraints}

\assnoncancellation*

\vspace{6pt}
\lemdiagterms*
\begin{proof}
    \keypoint{\textit{(Preliminaries)}}
    Recall $p(z)=\gN(0,\mI)$ and $p(x\!\mid\! z)=\gN(x;d(z),\sigma^2\mI)$. Let $q(z\!\mid\! x)=\gN(z;e(x),\Sigma_x)$ be the trained posterior with $\Sigma_x$ diagonal (by assumption). 
    Denote the SVD of the decoder Jacobian by
    \begin{gather}
        \mJ_z = \mU_z\,\mS_z\,\mV_z^\top,\quad 
        \mU_z\in\R^{m\times d},\quad 
        \mV_z\in\R^{d\times d}, 
        \nonumber\\
        \mS_z=\Diag(s_1(z),\ldots,s_d(z)),\quad s_i(z)>0.
    \end{gather}
    Assume full column rank on the manifold ($s_i(z)>0$).
    % Define the residual  $r(x,z) = x - d(z)$ and \emph{directed Hessian}
    % \[
    % $\big[\mH_{\mathrm{dir}}(z;x)\big]_{ij}
    % = \sum_{k=1}^n r_k(x,z)\,\frac{\partial^2 d_k(z)}{\partial z_i\,\partial z_j}$
    % \]
    % 
    For Gaussian likelihood with variance $\sigma^2$, the Hessian of the log-likelihood w.r.t.\ $z$ can be written
    % \begin{equation}\label{eq:like-hess}
    % 
    $$\nabla_z^2 \log p(x\!\mid\! z)
    \;=\;
    -\tfrac{1}{\sigma^2}\big(\mJ_z^\top \mJ_z 
    - \sum_{\ell=1}^n r(z)_\ell\, \mH_\ell(z)\big),$$
    % 
    % \end{equation}
    where $r(z)=x-d(z)$ and $\mH_\ell(z)\in\R^{d\times d}$ is the Hessian of the $\ell$-th decoder coordinate, $[\mH_\ell]_{pq}=\partial^2 d_\ell/\partial z_p\,\partial z_q$.
    Combined with the Opper–Archambeau fixed‑point yields
    \begin{align}\label{eq:oa}
        \Sigma_x^{-1}
            &\,=\, -\,\E_{q}\!\big[ \nabla_z^2\log p(x,z) \big] 
            \nonumber\\
            &\,=\, -\,\E_{q}\!\big[
                \nabla_z^2\log p(z) + \nabla_z^2\log p(x\!\!\mid\!\! z)
                \big]
            \nonumber\\
            &\,=\,\mI 
                - \E_{q}\!\big[\nabla_z^2 \log p(x\!\mid\! z)\big]
            \nonumber\\
            &\,=\,\mI 
                + \tfrac{1}{\sigma^2}\E_{q}\!\big[\mJ_z^\top\mJ_z\big]
                 - 
                \tfrac{1}{\sigma^2}\E_{q}\!\big[\sum_{\ell=1}^n r(z)_\ell\, \mH_\ell(z)\big].
    \end{align}
    
    \vspace{-10pt}
    \keypoint{(\textit{C1})} 
    For diagonal $\Sigma_x$ in \Eqref{eq:oa} and $z$ concentrated around $e(x)$ under P\ref{prp:non-cancellation} (``no cancellation''),  we have
    \begin{equation}
    \label{eq:diagonal_terms}
    \mJ_z^\top\mJ_z
    \ \text{ is diagonal},
        \quad
    \sum_{\ell=1}^n r(z)_\ell\, \mH_\ell(z)
    \ \text{ is diagonal}.
    \end{equation}
    Diagonal $\mJ_z^\top \mJ_z\!\eq\Diag(s_1^2,\ldots,s_k^2)$ implies right singular vectors $\vv^i$ are the standard basis, up to signed permutations. By relabelling latent axes and absorbing signs,
    $\mV_z=\mI$ for all $z$ visited by the encoder, establishing \textbf{C1}. 
    
    \keypoint{\textit{(Directed Hessian is Tangent to Manifold)}}
    For a trained model, $d(e(x))\cl\approx x$ is assumed to be small. For $z$ concentrated around $e(x)$, i.e.\ $z=e(x)+\delta$ with $\delta\cl>0$ small, a first‑order Taylor expansion gives
    \begin{align}        
        d(z) &\,=\, d(e(x)) + \mJ_{e(x)}\,\delta + O(\|\delta\|^2)
        \nonumber\\
        \Rightarrow\quad
        r(z) &\,\doteq\,x-d(z)\;\approx\;-\mJ_{e(x)}\,\delta + O(\|\delta\|^2).
    \end{align}
    Thus, to first order, for $z$ concentrated around $e(x)$, $r(z)$ lie in the column space of $\mJ_{e(x)}$, i.e.\ in the span of the left singular vectors $\{\!u_i(z)\!\}_{i=1}^k$ (columns of $\mU_z$):
    % \vspace{-8pt}
    \begin{equation}
    \label{eq:r-tangent}
    r\;=\; \mU_z a,\qquad a\in\R^k
    \end{equation}
    
    Hence, \Eqref{eq:diagonal_terms} implies that $\sum_{\ell=1}^n r_\ell\, \mH_\ell(z) $ is diagonal for all $r\clin\text{span}(\mU_z)$. In particular,
    % , for $z$ in a concentrated region around $e(x)$.
    % 
    $\sum_{\ell=1}^n u^\top_{i\ell}\, \mH_\ell(z)$ is diagonal for all rows $u_i$ of $\mU_z$ and, by definition of slices $\mH_\ell$ and $\mJ_z$,
    % \[
    \begin{equation}
        \label{eq:hessian_equivs}
        [\sum_{\ell=1}^n u^\top_{i\ell}\, \mH_\ell(z)]_{p q} 
         \ =\  \sum_{\ell=1}^n u^\top_{i\ell}\, \pderiv{^2d_\ell}{z_p z_q} 
         \ =\  (\mU_z^\top \pderiv{\mJ_z}{z_p})_{iq}    
    \end{equation}

    \vspace{-10pt}
    \keypoint{\textit{(Diagonality of $(\bm{\frac{\partial s_i}{\partial z_j})_{i,j}}$})}
    
    Differentiating 
        $\mU_z^\top \mU_z \eq \mI$ 
    to get
        $\tfrac{\partial \mU_z^\top}{\partial z_j}\mU_z \cl+ \mU_z^\top \tfrac{\partial \mU_z}{\partial z_j} \eq \mathbf{0}$
    shows 
        $\Omega_j(z) \cl{:=} \mU_z^\top\tfrac{\partial \mU_z}{\partial z_j} \clin \R^{d\times d}$ 
    is skew-symmetric.
    Differentiating $\mJ_z \eq \mU_z\,\mS_z$ w.r.t.\ $z_j$ and premultiplying by $\mU_z^\top$ then gives
    % gives
    \begin{equation}\label{eq:UtdJ}
    \mU_z^\top\tfrac{\partial \mJ_z}{\partial z_j}
    \;=\;
    \Omega_j(z)\,\mS_z \;+\; \tfrac{\partial \mS_z}{\partial z_j}.
    \end{equation}
    Since $\Omega_j(z)$ is skew-symmetric, all diagonal entries are zero; since $\tfrac{\partial \mS_z}{\partial z_j}$ is diagonal, all non-diagonal entries are zero. Thus, of respective entries, only one is non-zero and can be considered separately. 
    
    From \Eqref{eq:hessian_equivs}, only $(\Omega_j(z)\,\mS_z)_{:j}$ elements can be non-zero, thus \textit{all} elements of  $\Omega_j(z)$ must be zero (by skew-sym.), \textbf{ruling out rotation in the tangent plane}\footnote{
        we have: $\Omega_j(z)_{kj} = -\Omega_j(z)_{jk} = 0$, if $j\cl\ne k$;  and $\Omega_j(z)_{jj} = 0$ (skew-sym.).}
    
    For $\tfrac{\partial \mS_z}{\partial z_j}$, diagonality (only $(\tfrac{\partial \mS_z}{\partial z_j})_{kk}$ elements non-zero) and \Eqref{eq:hessian_equivs} (only $(\tfrac{\partial \mS_z}{\partial z_j})_{:j}$ elements non-zero) imply that only elements $(\tfrac{\partial \mS_z}{\partial z_j})_{jj}=\pderiv{s_j}{z_j}$ can be non-zero, eliminating mixed partials  
    \[
    \tfrac{\partial s_i}{\partial z_j}(z) \;=\; 0 \quad\text{for all } i\neq j,
    \]
    i.e.\ the Jacobian of the singular-value map $s(z)=(s_1(z),\dots,s_k(z))$ is \textit{diagonal}, proving \textbf{C2}.    
\end{proof}

% \clearpage
\begin{figure*}[!htbp]
    % \vspace{6pt}
    \centering
    \begin{subfigure}{0.85\textwidth}
        \centering
        \hspace{-20pt}
        \includegraphics[width=\linewidth]{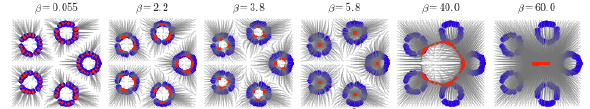}
        % \scalebox{0.99}{ \includegraphics[width=\linewidth]{linear2.png} }
    \end{subfigure}%
    \begin{subfigure}{.2\textwidth}
        \centering
        \hspace{-30pt}
        \scalebox{1.}{\includegraphics[width=0.8\linewidth]{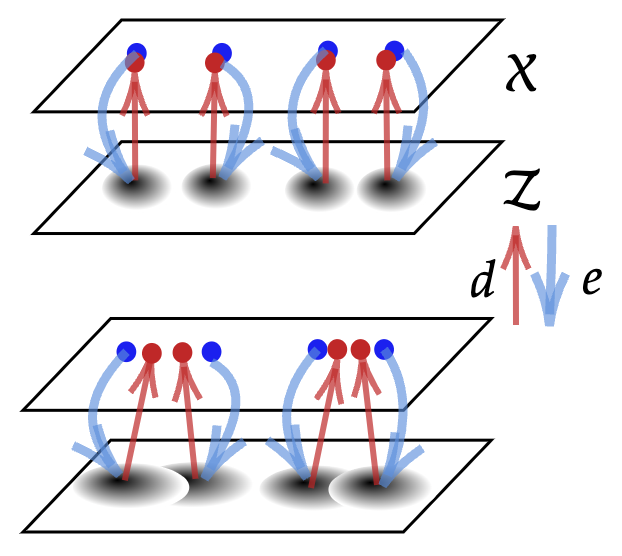}}
    \end{subfigure}%
    % \vspace{4pt}
    \caption{Illustrating $\beta\cl\propto\Var[x|z]$ (blue $\eq$ data, red $\eq$ reconstruction):
    (\textit{l}) For low $\beta$ ($\beta\eq0.55$), $\Var[x|z]$ is low (by \Eqref{eq:real_relationship}), and data must be well reconstructed (right, top). As $\beta$ increases, $\Var[x|z]$ and so $\Var[z|x]$ increase, and posteriors of nearby samples $\{x_i\}_i$ increasingly overlap (right, bottom). For $z$ in overlapping $\{q(z|x_i)\}_i$, the decoder $\E[x|z]$ maps to a weighted average of $\{x_i\}_i$. Initially, close neighbours reconstruct to their mean ($\beta\eq 2.2, 3.8$), then small circles ``become neighbours'' and map to their centres. Finally ($\beta\eq 60$), all samples reconstruct to the global centroid. \citep[reproduced with permission from][]{rezende2018taming}
    (\textit{r}) illustrating posterior overlap, (\textit{t}) low $\beta$, (\textit{b}) higher $\beta$.
    }
    
    \label{fig:taming}
    \vspace{-12pt}
\end{figure*}

\section{$\beta$ controls Noise Variance}
\label{app:beta}
Choosing $\beta\cl>1$ in \Eqref{eq:elbo} can enhance disentanglement \citep{betavae, burgess2018understanding} and has been viewed as re-weighting ELBO components or as a Lagrange multiplier. We show that $\beta$ implicitly controls the likelihood's  variance and that the ``$\beta$-ELBO'' remains a valid objective.

Dividing the ELBO by a constant and suitably adjusting the learning rate leaves the VAE training algorithm unchanged, hence consider \Eqref{eq:elbo} divided through by $\beta$ with the log likelihood scaled by $\beta\inv$. For a Gaussian VAE with $\Var[x|z]\eq\sigma^2$, this exactly equates to a standard VAE with variance $\beta\sigma^2$ \citep{lucas2019don}.
More generally, scaling the log likelihood by $\beta\inv$ is equivalent to an \textit{implicit likelihood} $p_\theta(x|z)^{^{_1}\cl/_\beta}$\!, where $\beta$ acts as a \textit{temperature} parameter: 
    $\beta\cl\to\infty$ increases the effective entropy towards uniform (the model assumes more noise in the data, fitting more loosely), and $\beta\cl\to 0$ reduces it to a delta (reconstructions should be tight). Optimal posteriors fit to the implicit likelihood,
$q_\phi(z|x) \propto p_\theta(x|z)^{^{_1}\cl/_\beta}p(z)$,
which thus dilate ($\beta>1$) or concentrate ($\beta<1$). This generalises the Gaussian result \citep{lucas2019don}, showing that the $\beta$-ELBO is simply the ELBO for a different likelihood model.\footnote{
    Technically, the $\beta$-ELBO's value is incorrect without renormalising the implicit likelihood, but that is typically irrelevant, e.g.\ for commonly used Gaussian likelihoods, only the quadratic "MSE" term appears in the loss.} 
    
\keypoint{Empirical support} Our claim, in effect that $\Var[x|z]\cl\propto\beta$, is well illustrated on synthetic data in \Figref{fig:taming} \citep[][see caption for details]{rezende2018taming}. It also immediately explains \textit{blur in $\beta$-VAEs} since $\beta\cl>1$ simply assumes more noise.
It also explains why $\beta\cl<1$ helps mitigate \textit{posterior collapse} \citep{bowman2015generating}, i.e.\ when a VAE's likelihood is sufficiently expressive that it can directly model the data distribution, $p(x|z) \eq p(x)$, leaving latent variables redundant (posterior ``collapses'' to prior). 
As $\beta\cl\to0$, the effective variance of $p_\theta(x|z)$, and the distributions it can describe, reduces. Thus for some $\beta\cl<1$ the effective variance falls below $\Var[x]$, rendering posterior collapse impossible as some variance in $x$ can only be explained by $z$.
Thus our claim that $\beta$ controls effective variance explains well-known empirical observations, which in turn provide empirical support for the claim.

\section{Identifiability Proofs}

\subsection{Proof of Linear VAE Identifiability}\label{app:proof_identifiability_LVAE}

\corlinear*

\vspace{-12pt}
\begin{proof}
\textbf{\textit{(Ground truth)}} If $x\eq \mW z$, then letting $u\cl\doteq \mU^\top x$ and $v\cl\doteq\mV^\top z$, gives $u\eq \mS v$, $u_i\eq s_iv_i$. 
Since $z\!\sim\!\gN(0,\mI)$ and $\mV$ is orthonormal, $v\!\sim\!\gN(0,\mI)$, 
hence $\{v_i\}$, and $\{u_i\}$ are mutually independent: $u_i\!\sim\! \gN(0,s_{i}^2)$ and 
$p_\mu^{\text{(g)}}(x) \eq \prod_{i=1}^d \gN(u_i;0,s_{i}^2)$.

\textbf{\textit{(Model)}} Let $d(z)\eq\mD z$ with SVD $\mD\eq \mU_{\mathsmaller\mD}\mS_{\mathsmaller\mD}\mV_{\mathsmaller\mD}^\top$. 
For an LVAE with diagonal covariances, the Hessian term in  \Eqref{eq:real_relationship} is zero, the expectation is redundant and P\ref{prp:non-cancellation} is trivially satisfied. Thus, by \Cref{lem:diag-terms}, s.v.\ paths  are axis–aligned (C1). Since  C2 is vacuously satisfied, \Cref{thm:non_linear} implies that $p_\mu^{\text{(d)}}$ is disentangled and factorises into statistically independent components along the decoder’s seams (columns of $\mU_{\mathsmaller\mD}$). 
Since $u_{\mathsmaller\mD,i}=s_{\mathsmaller\mD,i} z_i$ and $z_i\sim\gN(0,1)$, each seam factor is Gaussian with variance $s_{\mathsmaller\mD,i}^2$, i.e.\ 
$p_\mu^{\text{(d)}}=\prod_{i=1}^d \gN(u_{\mathsmaller\mD,i};0,s_{\mathsmaller\mD,i}^2)$.

\textbf{\textit{(Matching)}} Equality $p_\mu^{\text{(d)}} \equiv p_\mu^{\text{(g)}}$ and \emph{distinct} $\{s_{i}\}$ imply uniqueness of the Gaussian product decomposition, up to permutation. Thus the LVAE’s independent components (seam factors) match ground-truth components up to permutation/sign, i.e.\ identifiability and disentanglement on $\gM_g$.
\end{proof}

% \clearpage
\subsection{Proof of Gaussian VAE Identifiability}\label{app:intrinsic-uniqueness}

\lemintrinsicseams*

\begin{proof}
Fix $x\clin\gM_{\mathrm{reg}}$, and let $\vu^1,\ldots,\vu^d$ be the unit seam directions at
$x$, stacked as columns of $\mU_x\clin\R^{m\times d}$. By assumption, these directions
form an orthonormal basis of $T_x\gM$, with local seam coordinates
$u=(u_1,\ldots,u_d)$.

Since
\[
\log p_\mu(x)=\sum_{i=1}^d \log f_i(u_i(x)),
\]
its mixed second derivatives in the seam coordinates vanish:
\[
\big[\nabla_u^2\log p_\mu(x)\big]_{ij}
=
\begin{cases}
\tfrac{\partial^2}{\partial u_i^2}\log f_i(u_i(x)) & (i=j),\\[4pt]
0 & (i\neq j).
\end{cases}
\]

Because the seam coordinates are orthonormal, the manifold metric in these
coordinates is the identity, so the intrinsic Hessian equals the ordinary Hessian
in the $u$-coordinates at $x$.\footnote{Here “orthonormal seam coordinates” means that the seam coordinate vector fields form an orthonormal coordinate frame on the local patch.} Hence
\[
\tH_x
\;=\;
\mU_x\, [\nabla_u^2\log p_\mu(x)]\, \mU_x^\top
\]
is an eigendecomposition of $\tH_x$.

Therefore the seam directions are eigenvectors of $\tH_x$. Since the eigenvalues are
pairwise distinct a.e., these directions are unique up to permutation and sign.
\end{proof}

\paragraph{Implication for identifiability.}
\Cref{thm:intrinsic-uniqueness} isolates the intrinsic geometry of $p_\mu$: once $p_\mu$ is fixed, the seams and their directions are fixed (P\&S). Any Gaussian VAE decoder $d$ matching $p_\mu$ and satisfying \textbf{C1–C2} must therefore align its singular paths with those seams and inherit the same seam factors.

\lemdecoderfindsseams*

\begin{proof}
Since $p_\mu^{(d)}\equiv p_\mu$, both define the same manifold density on the common
manifold $\gM$.

Under \textbf{C1–C2}, \Cref{thm:seam-factorisation-general} gives
\[
p_\mu(x)
=
\prod_{i=1}^d \frac{p_i(z_i)}{s_i(z)},
\qquad x=d(z),
\]
with each $s_i(z)$ depending only on $z_i$ by \textbf{C2}. Define local seam
coordinates by
\[
u_i(z_i)\doteq \int_0^{z_i} s_i(\zeta)\,d\zeta .
\]
Then each factor depends only on the corresponding seam coordinate:
\[
p_\mu(x)=\prod_{i=1}^d f_i(u_i(x)),
\qquad
f_i(u_i)=\frac{p_i(z_i)}{s_i(z_i)}.
\]

Moreover, by the chain rule and \textbf{C1},
\[
\frac{\partial x}{\partial u_i}
=
\frac{\partial x}{\partial z_i}\frac{\partial z_i}{\partial u_i}
=
\frac{\mJ_z \vz_i}{s_i(z_i)}
=
\vu_i(z).
\]
Hence the seam coordinates induced by $d$ are orthonormal, since
\[
\Big\langle \frac{\partial x}{\partial u_i},\frac{\partial x}{\partial u_j}\Big\rangle
=
\vu_i(z)^\top \vu_j(z)
=
\delta_{ij}.
\]

Thus the assumptions of \Cref{thm:intrinsic-uniqueness} hold, with seam frame
$\mU_z$. By that lemma, the columns of $\mU_z$ are exactly the intrinsic seam
directions, up to P\&S. By \Cref{lem:seams}, singular-vector paths map to seams; and
by \Cref{thm:seam-factorisation-general}, the factor along seam $i$ is the 1-D
push-forward of $p_i(z_i)$.
\end{proof}

Note that the two proofs above adopt a similar technique, but \Cref{thm:intrinsic-uniqueness} is entirely intrinsic to the manifold (hence no mention of a Jacobian), whereas \Cref{cor:intrinsic-uniqueness} is with reference to a parameterisation of the manifold by a function $d$.

\cornonlinear*
% \subsection{Proof of Gaussian VAE Identifiability}
% \label{app:proof_identifiability_GVAE}

% \begin{corollary}[\textbf{Gaussian VAE Identifiability}]\label{cor:identifiability-succinct}
% Let data be generated by \emph{c.i.d.a.e.}\ $g:\gZ\!\to\!\gX$ with factorised prior $p(z)\eq\prod_{i=1}^d p_i(z_i)$; 
% let a Gaussian VAE with \emph{diagonal} posteriors learn a decoder $d\cl:\gZ\!\to\!\gX$.
% Suppose both $g$ and $d$ satisfy \textbf{C1–C2} \textit{a.e.}\ and  manifold densities match, 
% $p_\mu^{(d)}\!\!\equiv\! p_\mu^{(g)}\!$, on the common manifold $\gM\eq\{g(z)\}\eq\{d(z)\}$. 
% If the tangent Hessian $\tH_x$ has a simple spectrum \textit{a.e.}, then 
% $d$ identifies the ground‑truth seam decomposition (independent components) of $p^{(g)}_\mu$, up to P\&S.
% \end{corollary}

\begin{proof}
    (\textbf{Matching $\bm{\mU}$}) Equality $p_\mu^{(g)} \!\equiv p_\mu^{(d)}$ implies the same $\tH_x$. By \Cref{thm:intrinsic-uniqueness}, its eigenvectors are intrinsic and unique (P\&S), so $\mU^{(d)}_z \!= \mU^{(g)}_z$ (P\&S, herein assume indices are relabelled to match). 
    
    (\textbf{Matching $\bm{\mS}$})  In this common basis, under \textbf{C1–C2}, on‑manifold scores are equal:
    \begin{align*}
        \big[\nabla_u\log p_\mu^{(g)}(x)\big]_i 
            \;&=\; \tfrac{1}{s_i^{(g)}(z)}\tfrac{\partial}{\partial z_i}\!\Big(\log p_i(z_i)-\log s_i^{(g)}(z)\Big)
        \nonumber\\
            \;&=\;
        \big[\nabla_u\log p_\mu^{(d)}(x)\big]_i,
    \end{align*}
    hence integrating along seam $i$ (with $z_{\neg i}$ fixed, using equality of $p_\mu$ at a reference point
    % on the seam remark
    to fix the integration constant), 1-D seam factors $\frac{p_i(z_i)}{s_i}$ match. Since $p_i$ is fixed, $s^{(d)}_i\! \eq s^{(g)}_i\!$, i.e.\ $\mS^{(d)}_z \!\eq \mS_z^{(g)}$\!. 
    
    With $\mV^{(d)}_z \eq \mV^{(g)}_z \eq \mI$  by \textbf{C1}, it follows that $\mJ^{(d)}_z\!=\mJ^{(g)}_z$ and the seam decomposition is identified up to P\&S.
\end{proof}

\section{Disentanglement Metrics}\label{app:metrics}

    \keypoint{Axis alignment score (AAS)} Given a matrix of mutual information values, between each latent co-ordinate and each ground truth factor, one can normalise over rows or columns to compute a ``distribution'' of mutual information. 
    
    The entropy of each distribution gives a measure of how narrowly or sparsely information about a ground truth factor is captured across latents or the spread of information about each factor captured by a single latent. In either case, a ``high entropy'' distribution means information is widely spread, while low entropy means information about a factor is concentrated in a single latent, i.e.\ disentangled. 
    
    Entropy of the mutual information distribution can be computed row-wise or column-wise. AAS is a holistic metric combining the intuitions of both options into a single, robust score that evaluates how close matrix M is to a permuted diagonal form (zero entropy, perfect disentanglement).
    
    In a perfectly disentangled MI matrix, the sum of peak values per row equals the sum of peak values per column, and both equal the total sum of the matrix. AAS measures the ratio of the "sum of peaks" to the "total sum":
    \vspace{-4pt}

    \begin{verbatim}
    sum_col_max = sum(max(mut_info, dim=0))
    sum_row_max = sum(max(mut_info, dim=1))
    aas = 0.5 * (sum_row_max + sum_col_max) / sum(mut_info)
    \end{verbatim}

    \vspace{-10pt}
    \keypoint{Normalised off diagonal}
        For gradient terms (here, a $d\times d$ matrix $M$)  we compute a measure of diagonality by computing the ratios of normalised off-diagonal absolute values to on-diagonal values.
    \vspace{-4pt}

    \begin{verbatim}
    d = M.shape[1]
    num_off_diag = d * (d - 1)
    M = abs(M)
    M = diag(M)^(-0.5) * M * diag(M)^(-0.5)       # normalise
    mean_off_diag = (sum(M) - sum(diag(M))) / num_off_diag
    \end{verbatim}
    \vspace{-10pt}

% \clearpage
\section{Empirical results on Natural Data (CelebA)}
\label{app:celeba}

% \paragraph{Setup and purpose.}
This appendix complements \Secref{sec:empirical}, applying the same architecture and training regime used for dSprites (and as in \citet{burgess2018understanding}) to CelebA, a more complex natural dataset without introducing new confounds.
We vary $\beta$ and compare \emph{diagonal} vs \emph{full} posterior covariances, reporting: 
    (i) latent traversals showing the dependence of disentanglement on posterior structure (\Figref{fig:celeba_traversals});
    (ii) how utilisation of the latent space varies with $\beta$ and posterior structure (\Figref{fig:celeba_latent_info});
    (iii) diagonality of the Price/Bonnet derivative terms (\Figref{fig:celeba_heatmaps}); and
    (iv) reconstruction/sampling quality (\ref{fig:celeba_samples}). See captions for details.
\begin{figure*}[h!]
    % \vspace{-12pt}
    \centering
    % \begin{minipage}{0.8\textwidth}
    \captionsetup[subfigure]{labelformat=empty,skip=0.05\baselineskip}
    \subfloat[Diagonal posteriors]{\includegraphics[width=0.48\linewidth]{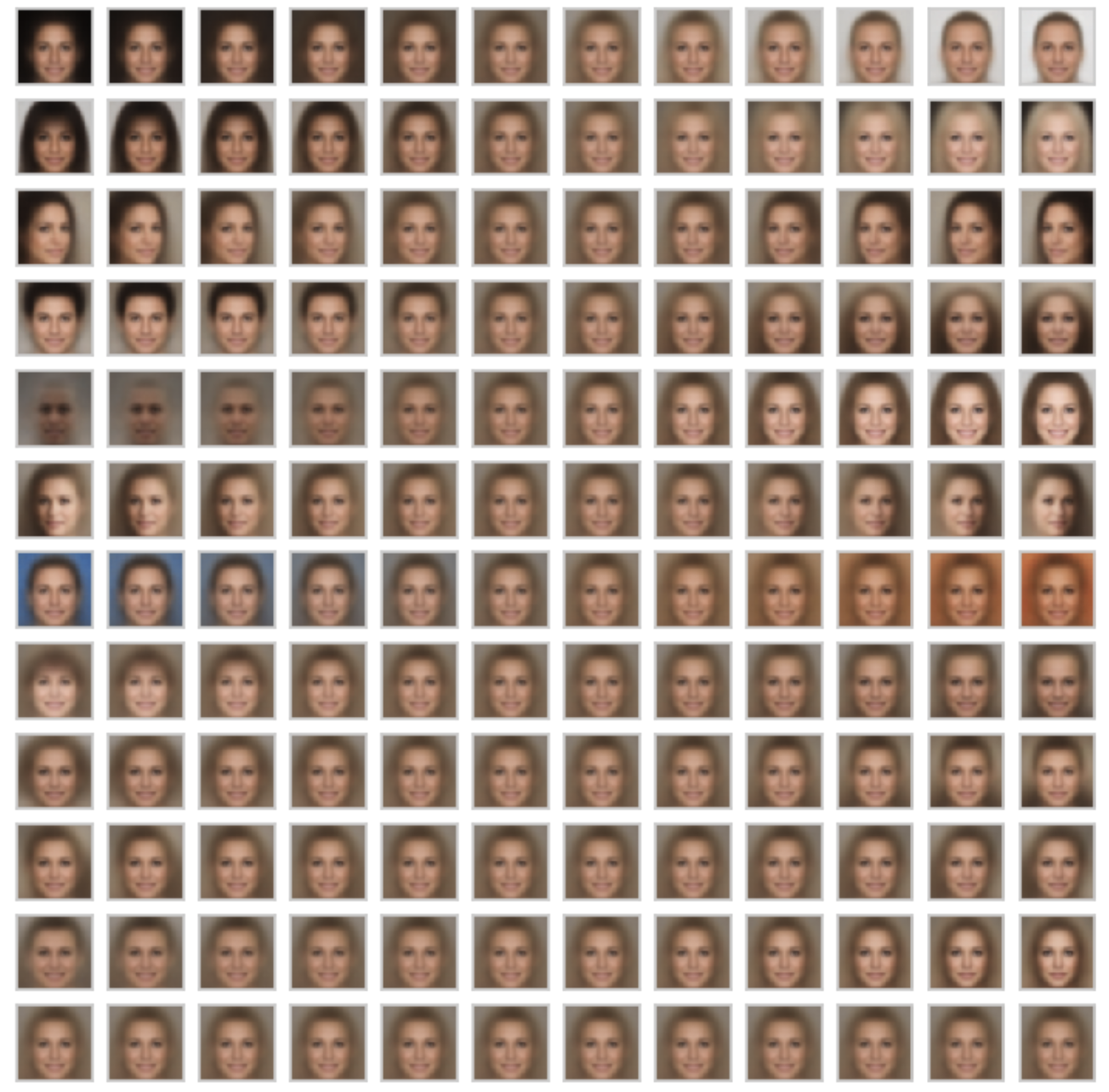}}%
    \hfill
    \subfloat[Full posteriors]{\includegraphics[width=0.48\linewidth]{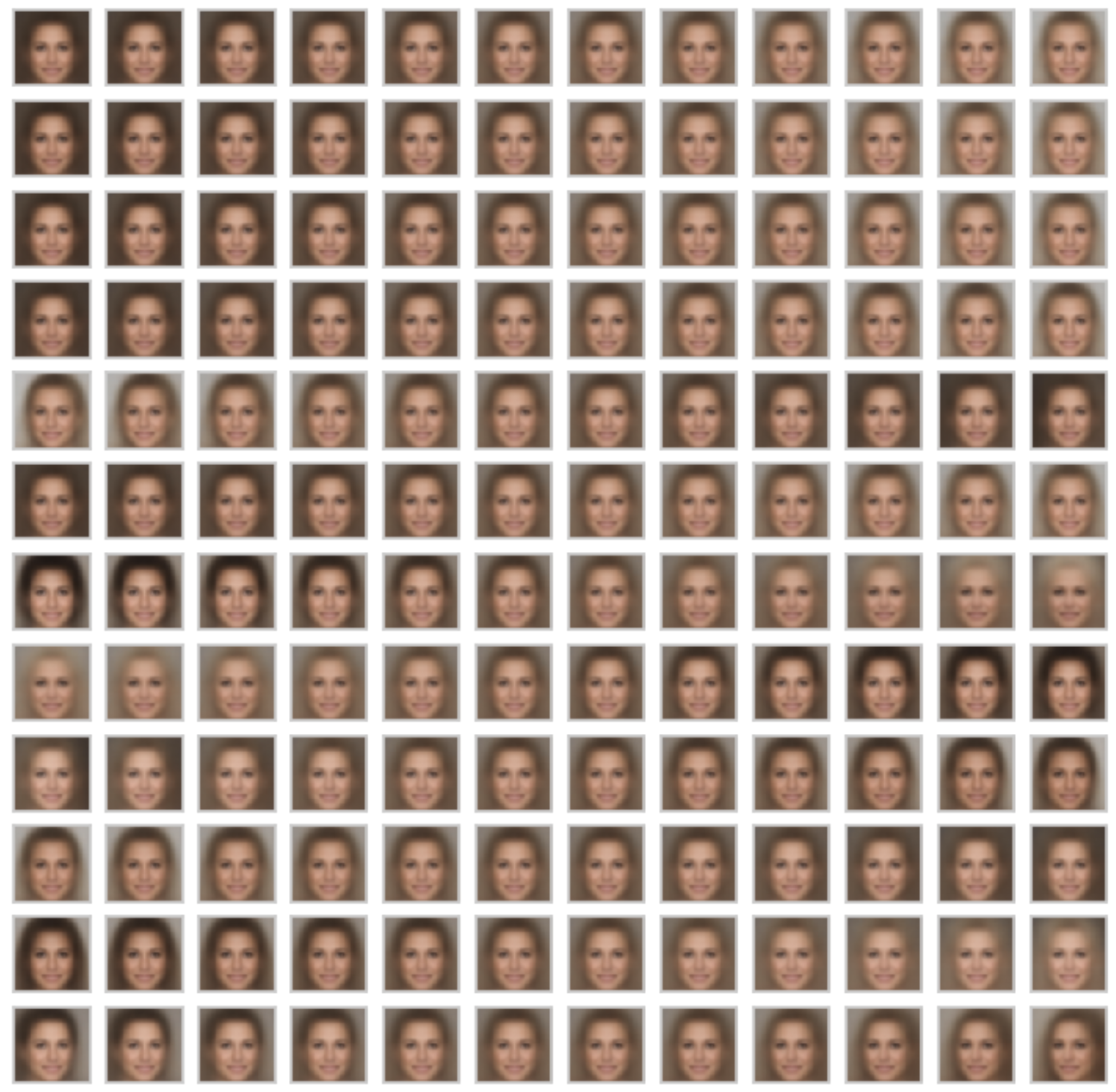}}%

    \vspace{-6pt}
    \caption{
         \textbf{Traversals over dimensions of highest variance} ($\beta\eq4$): Each row shows images generated as individual latent dimensions $z_i$ are varied with rows ordered by latent activity (See \Figref{fig:celeba_latent_info} caption). 
         For diagonal posteriors (left), traversals more clearly demonstrate disentanglement, i.e. identification of distinct semantic features with distinct latent dimensions, e.g.\ background shade (row 1), facial orientation (row 3), lighting (row 5), background colour (row 7). For full covariance posteriors, independent features are less clearly assigned to distinct latent dimensions, i.e.\ less disentangled.
         }
    \label{fig:celeba_traversals}
    % \vspace{-12pt}
\end{figure*}

\begin{figure*}[h!]
    \vspace{-12pt}
    \centering
    % \begin{minipage}{0.8\textwidth}
    \captionsetup[subfigure]{labelformat=empty,skip=0.05\baselineskip}
    \subfloat[Diagonal posteriors]{\includegraphics[width=0.48\linewidth, 
        trim={1 1 0 3},clip]{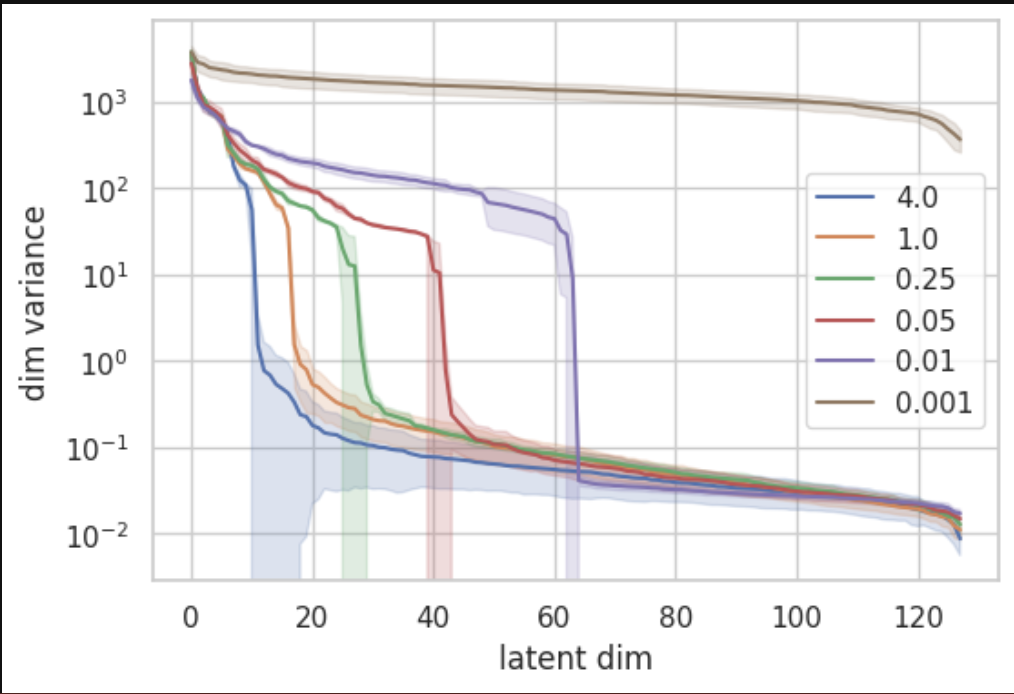}}%
    \hfill
    \subfloat[Full posteriors]{\includegraphics[width=0.48\linewidth,
        trim={0 0 1 0},clip]{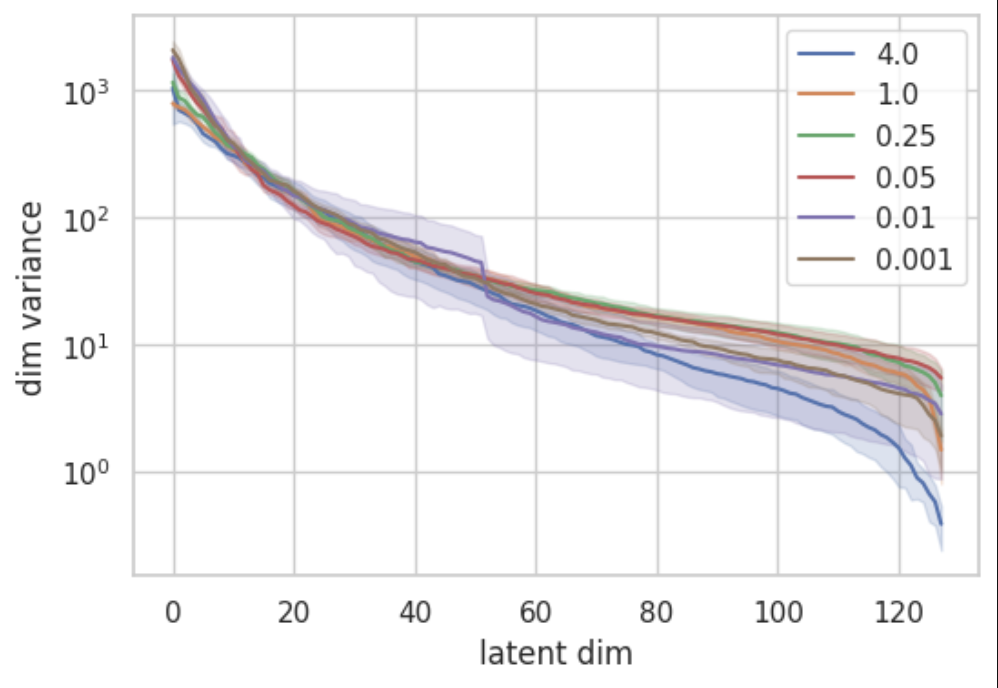}}%

    % \end{minipage}
    % \hfill
    % \begin{minipage}{0.19\textwidth}
    %     \centering
    %     % \vspace{-100}
    %     % \vspace{-30pt}
    %     \captionsetup[subfigure]{labelformat=empty,skip=0.25\baselineskip}
    %     \subfloat[\scriptsize \centering \,Diag\newline $\beta\!: 1\quad$]{\includegraphics[width=0.33\linewidth]{figures/JnH_beta_1_D.png}}%
    %     \subfloat[\scriptsize \centering \,Diag  \newline $1e\sminus3$]{\includegraphics[width=0.33\linewidth]{figures/JnH_beta_1e3.png}}%
    %     \subfloat[\scriptsize \centering \,Full\newline \!$1$]{\includegraphics[width=0.33\linewidth]{figures/JnH_beta_1_F.png}}%
    %     \vspace{4pt}
    %     \subcaption{\footnotesize 
    %         \centering (\textit{t}) $\smaller \mJ^\top_z\mJ$;  \newline
    %         (\textit{b}) $\smaller  (x\sminus d(z))\tH_z$ 
    %         } 

    % \end{minipage}
    \vspace{-6pt}
    \caption{
         \textbf{Active latent dimension depend on  $\bm\beta$ (denoted by colour) and posterior covariance structure (left/right)}: 
         For a trained model and for each latent dimension $z_i$, the variance $x|z_i$ is estimated by taking equidistant traversals in latent space and computing the Euclidean distance between samples at each end of the traversal. 
         The plot shows the (estimated) variance, or \textit{latent activity},  per dimension ordered by magnitude (log scale, mean over 5 runs, standard deviation indicated by shaded areas).
         % \\[4pt]
         With diagonal covariances (left), there are relatively sharp cliff-edges implying that dimensions are (broadly) \textit{active} or \textit{inactive} and that axis-aligned directions are preferred. 
         The number of active dimensions increases as $\beta$ reduces and reconstructions get ``sharper'', requiring more information to be captured and thus greater capacity (i.e.\ latent dimensions). 
         For full posteriors (right), the distribution of variance over dimensions is much smoother and axis-aligned directions have no special status.
         }
    \label{fig:celeba_latent_info}
    % \vspace{-12pt}
\end{figure*}

\begin{figure*}[h!]
    \vspace{-12pt}
    \hspace{-6.83cm}\small{$\beta$}\hfill\\
    \centering
    % \begin{minipage}{0.05\textwidth}
        % \vspace{-20pt}
        \small{4}
    % \end{minipage}
    \hfill
    % \captionsetup[subfigure]{labelformat=empty,skip=0.05\baselineskip}
    % \subfloat[Diagonal posteriors]{
    \includegraphics[width=0.45\linewidth]{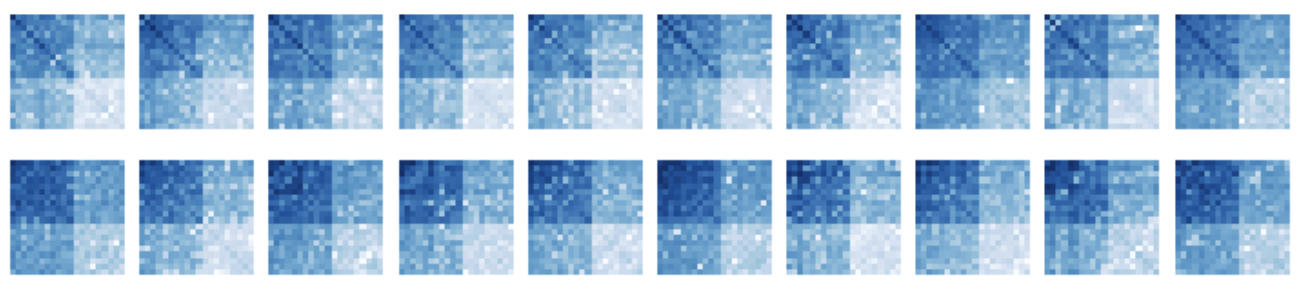}
    % }%
    \hspace{4pt} 
    % \subfloat[Full posteriors]{
    \includegraphics[width=0.45\linewidth]{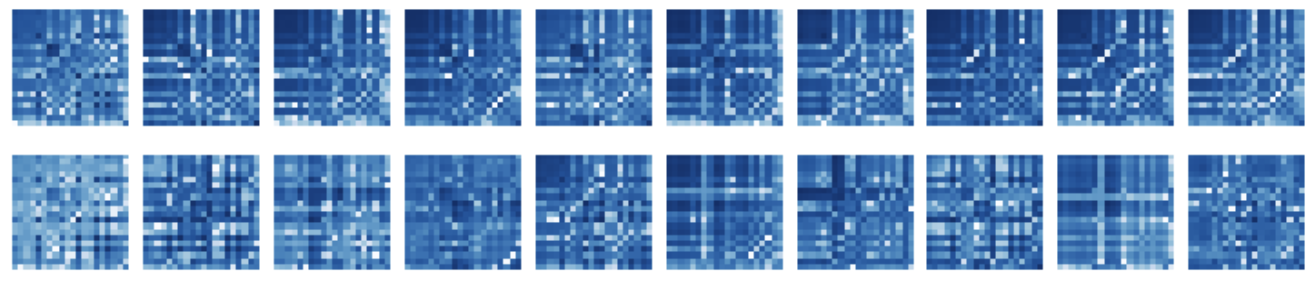}
    % }%
    \\
    \vspace{4pt}
    \small{1}\hfill
    % \captionsetup[subfigure]{labelformat=empty,skip=0.05\baselineskip}
    % \subfloat[Diagonal posteriors]{
    \includegraphics[width=0.45\linewidth]{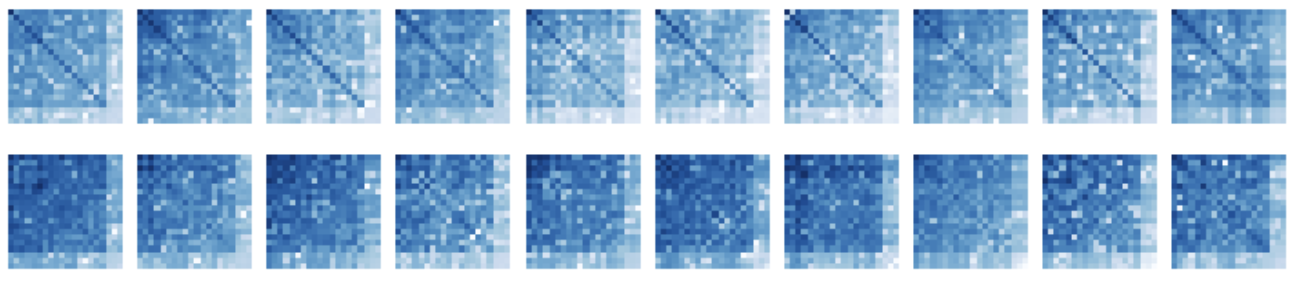}
    % }%
    \hspace{4pt} 
    % \subfloat[Full posteriors]{
    \includegraphics[width=0.45\linewidth]{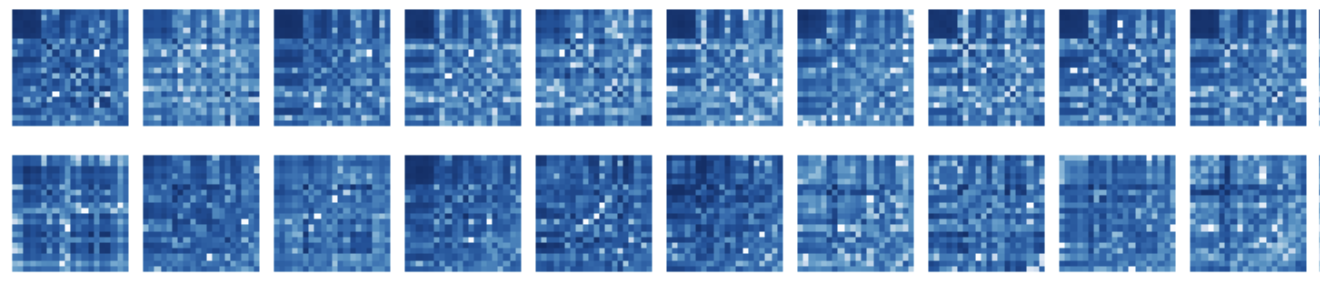}
    % }%
    \\
    \vspace{4pt}
    \small{0.05}\hfill
    \subfloat[Diagonal posteriors]{
    \includegraphics[width=0.45\linewidth]{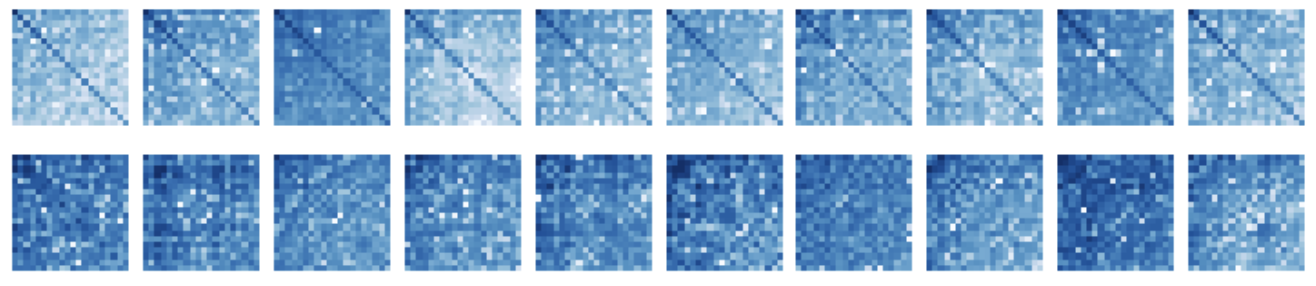}}%
    \hspace{4pt} 
    \subfloat[Full posteriors]{
    \includegraphics[width=0.45\linewidth]{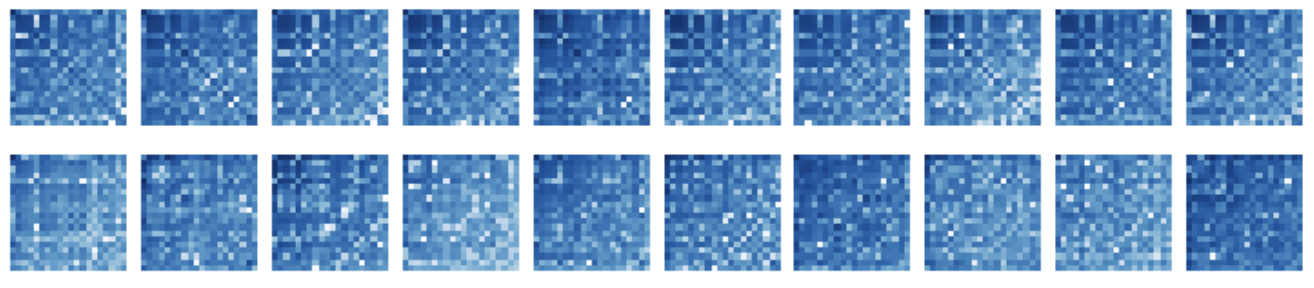}}%

    \vspace{-6pt}
    \captionsetup[subfigure]{labelformat=empty,skip=0.05\baselineskip}
    \caption{
         \textbf{Heatmaps of Derivative Terms in \Eqref{eq:real_relationship} for CelebA for Diagonal and Full Posterior Covariances}: 
         For each value of $\beta$ (indicated on left), there are two rows for $\mJ^\top\mJ$ (upper) and the directed Hessian (lower) over 10 random test samples. Colour intensity indicates log magnitude of matrix entries. For each heatmap, rows and columns correspond to a latent dimension $z_i$, ordered by latent activity (See \Figref{fig:celeba_latent_info} caption). We show the top 20 most active dimensions.
         For diagonal covariances (left), the active dimensions are visible as a darker block in the upper left, which grows as $\beta$ reduces (matching \Figref{fig:celeba_latent_info}), and diagonal structure is visible for active dimensions of $\mJ^\top\mJ$. Such structure is not visible for full covariances (right). The Hessians show less discernable structure and we suspect that such a higher-order derivative requires more samples to be well estimated for a complex distribution.
         }
    \label{fig:celeba_heatmaps}
    \vspace{-6pt}
\end{figure*}

\begin{figure*}[h!]
    \vspace{2pt}
    % \hspace{-6.83cm}\small{$\beta$}\hfill\\
    \centering
    \small{4}\hfill
    \includegraphics[width=0.9\linewidth]{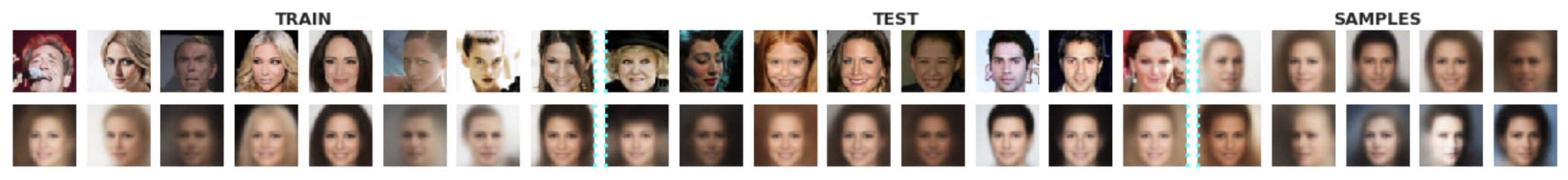}
    \\
    \small{1}\hfill
    \includegraphics[width=0.9\linewidth]{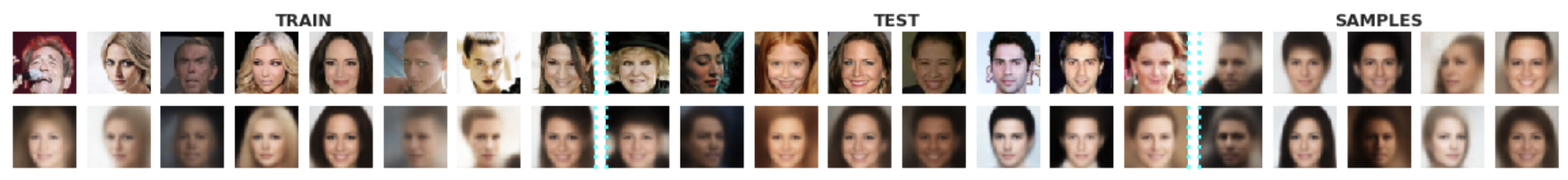}
    \\
    \small{0.05}\hfill
    \subfloat[Diagonal posteriors]{
    \includegraphics[width=0.9\linewidth]{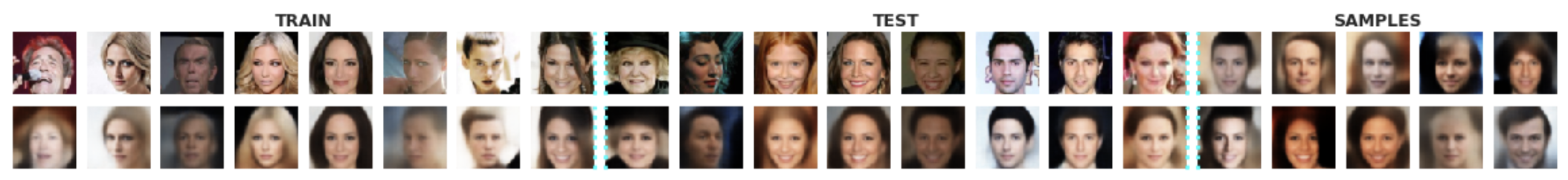}}%
    \\
% }%
    \small{4}\hfill
    \includegraphics[width=0.9\linewidth]{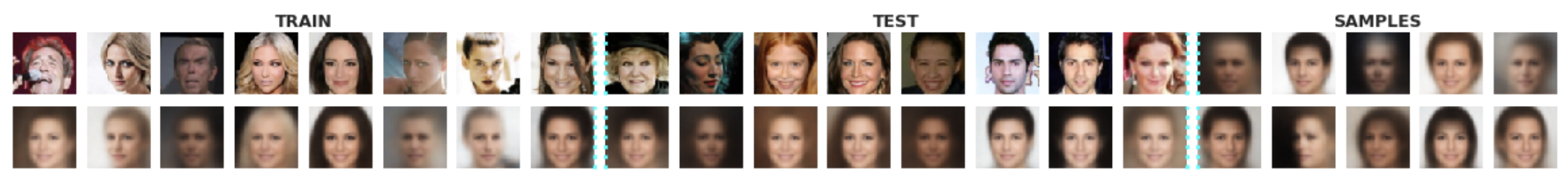}
    \\
    \small{1}\hfill
    \includegraphics[width=0.9\linewidth]{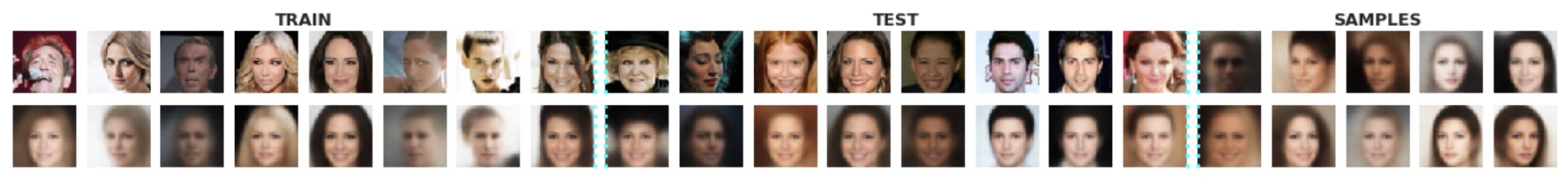}
    \\
    \small{0.05}\hfill
    \subfloat[Full posteriors]{
    \includegraphics[width=0.9\linewidth]{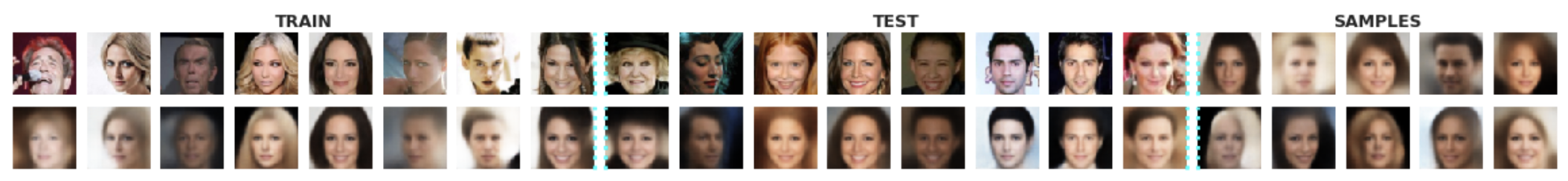}}%
    \vspace{-8pt}
    \captionsetup[subfigure]{labelformat=empty,skip=0.05\baselineskip}
    \caption{
         \textbf{Reconstructions and Samples for CelebA for a range of $\beta$ values and  Diagonal and Full Posterior Covariances}:
         For each value of $\beta$ (left), there are two rows in three sections: (left) train samples (upper) and reconstructions (lower); (mid) test samples (upper) and reconstructions (lower); and (right) samples (both rows). 
         As $\beta$ reduces, reconstruction quality and samples improve (i.e.\ blur reduces). 
         Reconstruction and sample quality is broadly comparable for diagonal and full covariances, indicating that the latent space is reoriented towards axis-alignment without necessarily impacting performance.
         }
    \label{fig:celeba_samples}
    \vspace{-14pt}
\end{figure*}

\paragraph{Summary:}
For \emph{diagonal} posteriors, we observe: (a) a small set of \emph{active} latents whose number increases as $\beta$ decreases; (b) stronger Jacobian orthogonality among active dimensions; and (c) reconstructions/samples of comparable quality to the full-covariance model across $\beta$ (Figures~\ref{fig:celeba_traversals}--\ref{fig:celeba_samples}).
These match the predictions of our theory and mirror the synthetic/dSprites trends, supporting on natural data the claim that diagonal posteriors drive C1--C2 in expectation.

\clearpage
\section{Reducing $\beta$ over training}\label{app:profile_results}
    \vspace{-58pt}

\begin{figure*}[h!]
    \vspace{-58pt}
    \centering
    \begin{minipage}{0.50\textwidth}
        \centering
        % \captionsetup[subfigure]{labelformat=empty,skip=0.05\baselineskip}
        \captionsetup[subfigure]{skip=0.05\baselineskip}
        \subfloat[$\beta \eq1$: good disentanglement, blurry reconstructions. ]{\includegraphics[width=1\linewidth]{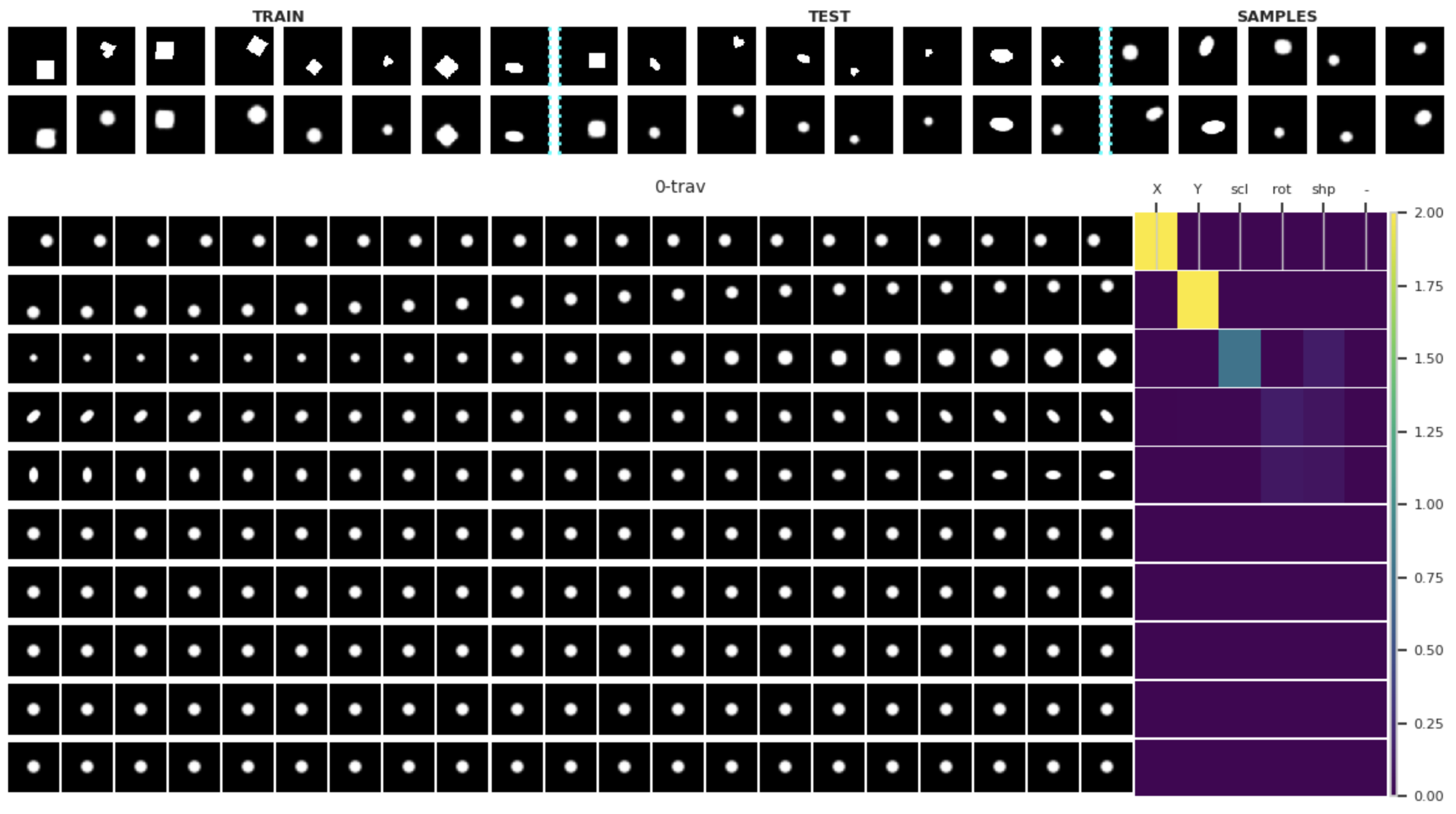}}%
        \vspace{6pt}
        \subfloat[$\beta \eq10^{\sminus3}$: no clear disentanglement, good reconstructions.]{\includegraphics[width=1\linewidth]{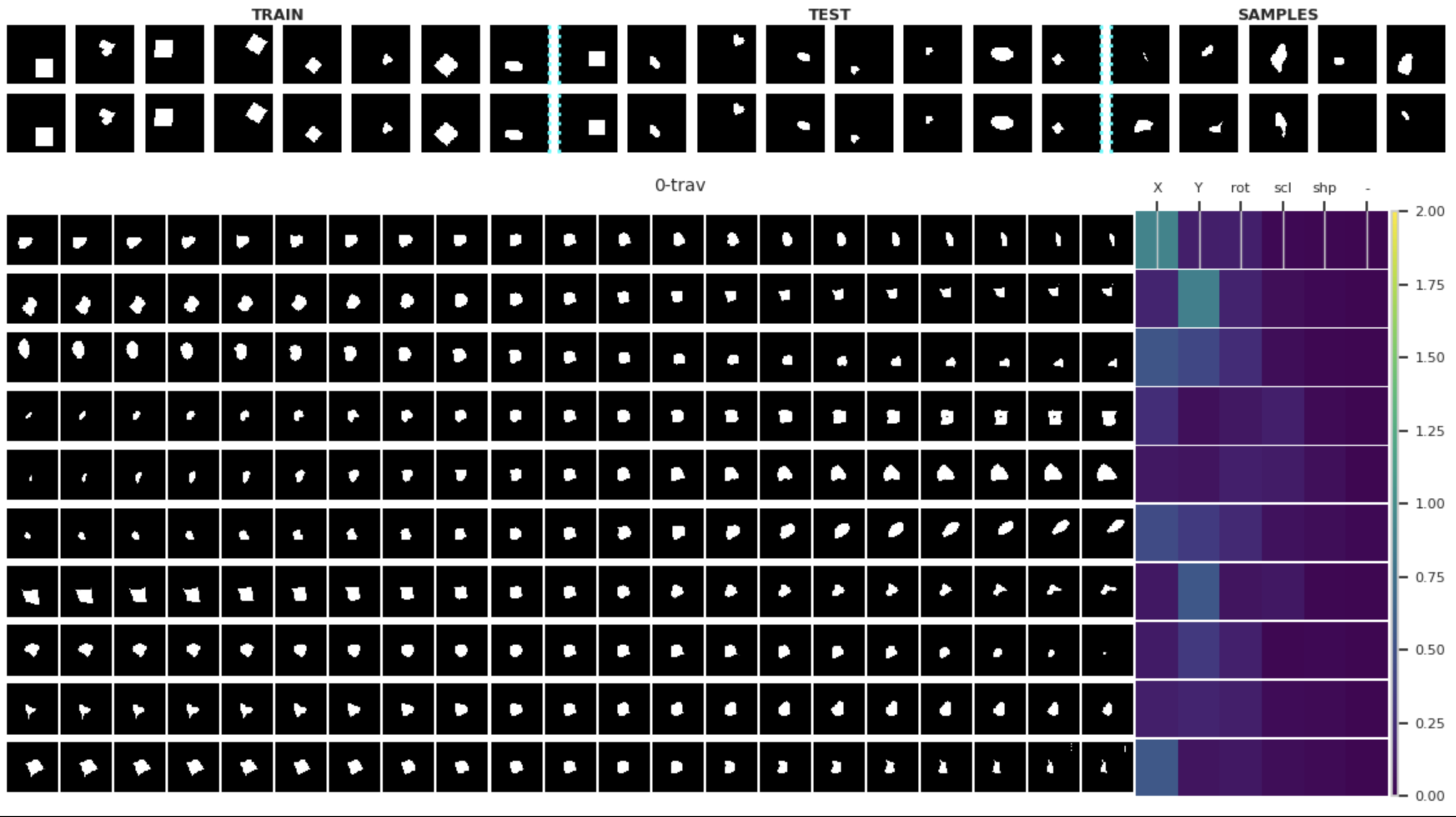}}%
        \vspace{6pt}
        \subfloat[$\beta \eq1$: good disentanglement, good reconstructions. ]{\includegraphics[width=1\linewidth]{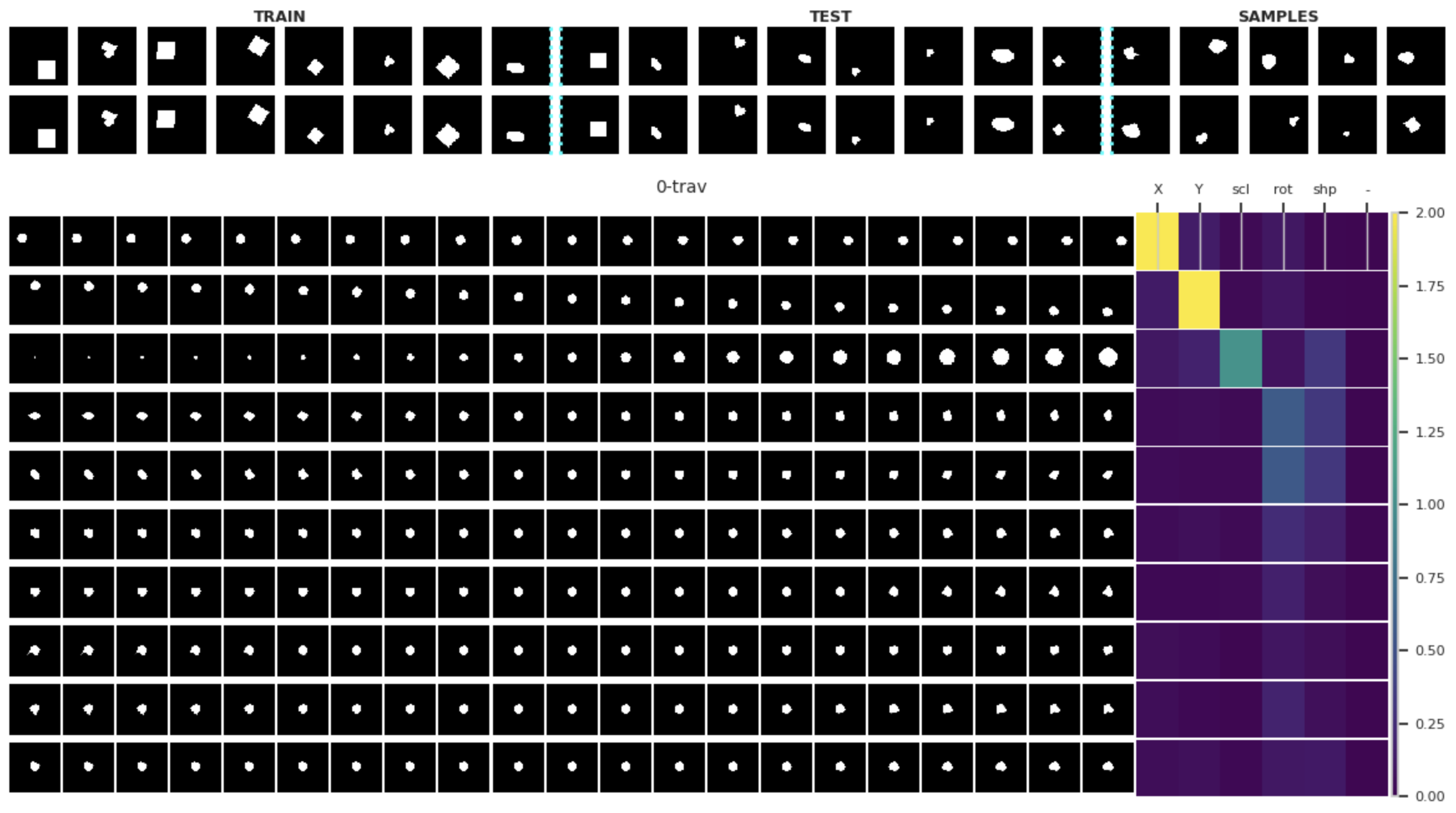}}%
        \vspace{-2pt}
        
        \hspace{-10pt}
        \subfloat[Traversals from a random test sample]{\includegraphics[width=0.96\linewidth]{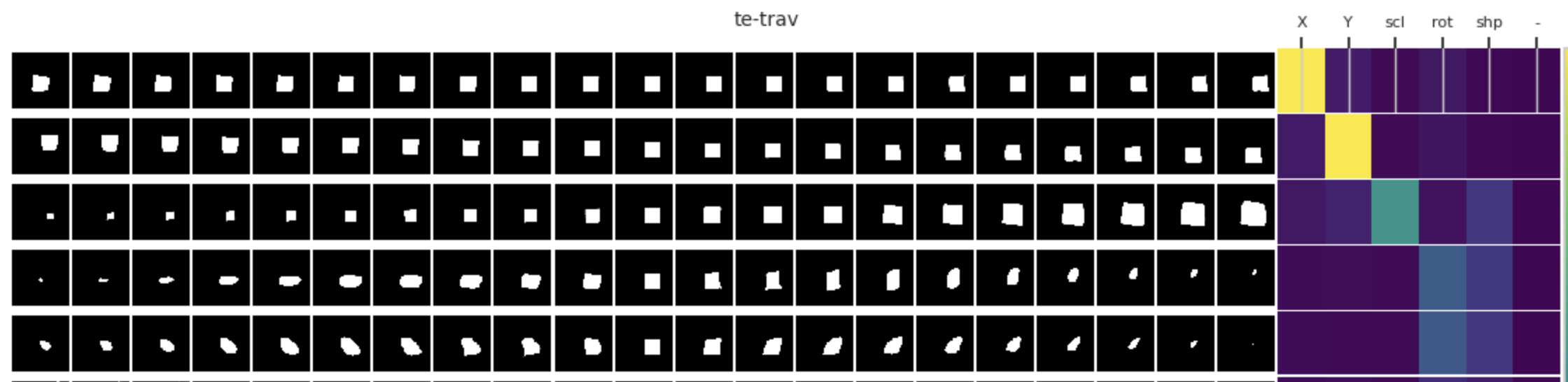}
        \hfill}%
    \end{minipage}

    \vspace{-2pt}
    \caption{
         \textbf{Testing the $\beta$-hypothesis:} (top) high $\beta$ (1) gives best disentanglement (see heatmap) but blurry images (see top rows); (mid) low $\beta$ (0.001) gives poor disentanglement but good reconstructions; (bottom) lowering $\beta$ over training ($1\cl\to0.001$) gives good disentanglement (see heatmap) and good reconstructions.
     }
    \label{fig:traversals}
    % \vspace{-10pt}
\end{figure*}

\clearpage
\section{Material Errors in \cite{reizinger2022embrace}}
\label{app:errors}

We note what appear to be several fundamental mathematical errors in the proof of Theorem 1 in \cite{reizinger2022embrace} rendering it invalid. Theorem 1 claims an approximation to the exact relationship given in \Eqref{eq:real_relationship}

\begin{enumerate}[leftmargin=0.5cm]
    \item p.33, after ``triangle inequality'': \  
        $\vline\, \E\big[\|a\|^2 \cl- \|b\|^2\big ]\,\vline 
            \, \leq\, \E \big[ \|a\cl-b\|^2 \big]$, 
            where $a \eq x \cl- f,\ \  b \eq \sminus\sum \pderiv{f}{z_k}...$
        \begin{itemize}[leftmargin=0.5cm]
            \item (dropping expectations for clarity) this has the form $\vline\, \|a\|^2 - \|b\|^2\,\vline  \ \leq\  \|a-b\|^2$  \qquad (*)
            \item true triangle inequality:\quad  
                $\vline\, \|a\|\cl-\|b\|\,\vline  \, \leq\,  \|a-b\| 
                    \ \implies 
                    |\|a\|\cl-\|b\||^2 \ \leq\  \|a - b\|^2$ (by squaring)
                \begin{itemize}
                    \item this differs to (*) since norms are squared inside the absolute operator on the L.H.S.
                \end{itemize}
            \item counter-example to (*):\quad
                $b\eq x>0,\ \ a\eq x\cl+1
                    \ \implies\ 
                    \vline \|a\|^2 - \|b\|^2\,\vline = |2x+1| \ >\ 1=
                    \|a-b\|^2$
        
        \end{itemize}
        
    \item next step, p.33: \quad
        $\E\big[  \|(c\cl-e) - (d\cl-e)\|^2\big]\ \leq\ \E\big[\|c\cl- e\|^2 + \|d\cl-e\|^2\big]
            \qquad$where $c\eq x,\ d\eq f(z) - \sum\pderiv{f}{z_k}...,\ 
                e\eq f(\mu)$
        \begin{itemize}[leftmargin=0.5cm]
            \item this has the form of the standard triangle inequality $\|a\cl- b\|\ \leq\ \|a\|\cl+\|b\|$ except all norms are squared.
            \item squaring both sides of the triangle inequality gives an additional cross term on the right that the used inequality omits, without which the inequality does not hold in general.
        \end{itemize}

    \item first step, p.34: drops the K term, which bounds the decoder Hessian and higher derivatives (in earlier Taylor expansion)
        \begin{itemize}[leftmargin=0.5cm]
            \item this omission is similar to a step in \cite{kumar2020implicit} but is not stated, e.g.\ in Assumption 1.
            \item since K is unbounded, any conclusion omitting it without justification is not valid in general.
        \end{itemize}
\end{enumerate}%%%%%%%%%%%%%%%%%%%%%%%%%%%%%%%%%%%%%%%%%%%%%%%%%%%%%%%%%%%%

\end{document}